\documentclass{article}

\usepackage{subfigure}
\usepackage[utf8]{inputenc} 
\usepackage[T1]{fontenc}    
\usepackage{hyperref}       
\usepackage{url}            
\usepackage{booktabs}       
\usepackage{amsfonts}       
\usepackage{nicefrac}       
\usepackage{microtype}      
\usepackage{graphicx}
\usepackage{FMTmicros}
\usepackage{longtable}
\usepackage{enumerate}
\usepackage{placeins}
\usepackage{tikz}
\usepackage{callouts}
\usetikzlibrary{arrows,decorations.pathmorphing,decorations.pathreplacing,backgrounds,positioning,fit,petri,graphs}
\usepackage[textfont=scriptsize,labelfont=bf]{caption}
\usepackage{multirow}
\usepackage{tabularx}
\usepackage{wrapfig}

\definecolor{xuebin}{rgb}{0.7, 0.4, 1.0}

\newcounter{bxincomm}
\definecolor{aqua}{rgb}{0.00,0.67,0.80}

\usepackage[accepted]{icml2021}

\icmltitlerunning{How Framelets Enhance Graph Neural Networks}

\begin{document}

\twocolumn[
\icmltitle{How Framelets Enhance Graph Neural Networks}



\icmlsetsymbol{equal}{*}

\begin{icmlauthorlist}
\icmlauthor{Xuebin Zheng}{equal,A}
\icmlauthor{Bingxin Zhou}{equal,A}
\icmlauthor{Junbin Gao}{A}
\icmlauthor{Yu Guang Wang}{B,C,G}\\
\icmlauthor{Pietro Li\`{o}}{D}
\icmlauthor{Ming Li}{E}
\icmlauthor{Guido Mont\'{u}far}{F,B}
\end{icmlauthorlist}

\icmlaffiliation{A}{The University of Sydney Business School, The University of Sydney, Camperdown, NSW 2006, Australia.}
\icmlaffiliation{B}{Max Planck Institute for Mathematics in the Sciences, Leipzig, Germany.}
\icmlaffiliation{C}{Institute of Natural Sciences and School of Mathematical Sciences, Shanghai Jiao Tong University, China.}
\icmlaffiliation{G}{School of Mathematics and Statistics, The University of New South Wales, Sydney, Australia.}
\icmlaffiliation{D}{Department of Computer Science and Technology, University of Cambridge, Cambridge, United Kingdom.}
\icmlaffiliation{E}{Key Laboratory of Intelligent Education Technology and Application of Zhejiang Province, Zhejiang Normal University, Jinhua, China.}
\icmlaffiliation{F}{Department of Mathematics and Department of Statistics, University of California, Los Angeles, United States}

\icmlcorrespondingauthor{Bingxin Zhou}{bzho3923@uni.sydney.edu.au}
\icmlcorrespondingauthor{Yu Guang Wang}{yuguang.wang@mis.mpg.de}

\icmlkeywords{Machine Learning, ICML}

\vskip 0.3in
]



\printAffiliationsAndNotice{\icmlEqualContribution} 

\begin{abstract}
This paper presents a new approach for assembling graph neural networks based on framelet transforms. The latter provides a multi-scale representation for graph-structured data. We decompose an input graph into low-pass and high-pass frequencies coefficients for network training, which then defines a framelet-based graph convolution. The framelet decomposition naturally induces a graph pooling strategy by aggregating the graph feature into low-pass and high-pass spectra, which considers both the feature values and geometry of the graph data and conserves the total information. The graph neural networks with the proposed framelet convolution and pooling achieve state-of-the-art performance in many node and graph prediction tasks. Moreover, we propose shrinkage as a new activation for the framelet convolution, which thresholds high-frequency information at different scales. Compared to ReLU, shrinkage activation improves model performance on denoising and signal compression: noises in both node and structure can be significantly reduced by accurately cutting off the high-pass coefficients from framelet decomposition, and the signal can be compressed to less than half its original size with well-preserved prediction performance.
\end{abstract}

\section{Introduction}
Graph neural networks (GNNs) are a powerful deep learning method for prediction tasks on graph-structured data \cite{wu2021comprehensive}. Most existing GNN models are spatial-based, such as GCN \cite{KiWe2017}, GAT \cite{velivckovic2017graph} and GIN \cite{xu2018how}. These methods compute graph convolution on vertices and edges in the form of message passing \cite{Gilmer_etal2017}, but leave the signal frequency of graph data unexplored. In this paper, we seek to exploit signal processing for GNNs. Graph framelets \cite{dong2017sparse,zheng2020decimated}, akin to traditional wavelets, provide a multiresolution analysis (MRA) for graph signals. The fully tensorized \emph{framelet transforms} guarantee an efficient design of graph convolution that combines low-pass and high-pass information, where the transforms only require graph Laplacian, Chebyshev polynomial approximation, and filter banks. We propose \emph{framelet convolution} that exploits the decomposition and reconstruction procedures of framelet transforms, and the network learns in the frequency domain.

The wavelet-like graph data analysis allows us to exploit traditional tools from signal processing. An effective practice is the \emph{shrinkage} that thresholds high-pass coefficients in the framelet representation. The multi-scale property of framelet convolution introduces diffusion for GNNs. The associated shrinkage defines a new type of activation that adapts to the diffusion scales in nonlinear feature transformation and extraction. Shrinkage in framelet convolution provides a mechanism for effective noise reduction of input graph data, where noise may appear in node and/or edge features. This ability is inherited from the traditional wavelet denoising model. Moreover, the shrinkage activation provides a way to compress graph signals. The shrinkage trims coefficients and significantly compresses the graph data representation while at the same time the underlying GNN reaches state-of-the-art performance in multiple tasks. The framelet MRA and shrinkage threshold provide GNNs with multi-scale and compression characteristics, which distinguishes our model from existing graph convolution methods.

The framelet transform naturally induces a graph pooling strategy by aggregating the different scales of framelet features. The consequent framelet pooling conserves the total information due to energy conservation of framelet spectra, and offers an efficient graph dimensionality reduction. Framelet-based GNNs outperform existing spatial or spectral-based methods on benchmark node and graph prediction tasks. Moreover, the framelet convolution with ReLU or shrinkage can achieve excellent performance in denoising both node features and graph structure. This behavior suggests that the framelets play an important role in bridging signal processing and graph neural networks.

\section{Related Works}
\paragraph{Graph Framelets and Transforms}
The construction of wavelet analysis on graphs was first explored by \citet{CrKo2003}. \citet{MaMh2008} used polynomials of a differential operator to build multi-scale transforms. 
The spectral graph wavelet transforms \cite{HaVaRe2011} define the graph spectrum from a graph's Laplacian matrix, where the scaling function is approached by the Chebyshev approximation. \citet{behjat2016signal} encoded energy spectral density to design tight frames on graphs with both graph topology and signal features. \citet{dong2017sparse} approximated piece-wise smooth functions with undecimated tight wavelet frames. Fast decomposition and reconstruction become possible with the filtered Chebyshev polynomial approximation and proper design of filter banks. 

The other regime of signal processing on graph data is backboned by the multiresolution analysis \cite{mallat1989theory} that establishes a tree-based wavelet system with localization properties. \citet{CrKo2003} defined the `h-hop' neighborhoods on binary graphs and \citet{GaNaCo2010} constructed Haar-like bases. The Haar-like orthonormal wavelet system \cite{ChFiMh2015} has been applied to deep learning tasks on undirected graphs \cite{wang2020haar,li2020fast,zheng2020graph}. Meanwhile, fast tight framelet filter bank transforms on quadrature-based framelets are explored on graph domain \cite{wang2019tight,zheng2020decimated} and manifold space \cite{WaZh2018} with a low redundancy rate.

\paragraph{Graph Convolution and Graph Pooling}
The theory of graph convolution \cite{bruna2013spectral} facilitated the later development of advanced deep learning methods. For example, spectral-based GNNs \cite{DeBrVa2016,xu2019graph,li2020fast,balcilar2021analyzing} transform graph signals to the spectral domain and process them with filter operations. Alternatively, spatial-based graph convolution performs node property prediction via aggregating feature information over neighborhood nodes \cite{KiWe2017,Gilmer_etal2017,wang2020gcn,Vignac2020,chen2020simple}.
For graph property prediction, one pursues topology-aware graph embedding via graph pooling. Some global pooling strategies refine vertex features in one-shot \cite{zhang2018end,lee2019self}, while others process graph information hierarchically \cite{cangea2018towards, gao2019graph, knyazev2019understanding,wang2020haar,Ma2020}.

\paragraph{Signal Compression and Denoising}
Signal compression is critical for high-speed signal transmission. Wavelets play an important role in compressing signal and have contributed to the prevalent JPEG 2000 \cite{jp2}. Our shrinkage framelet convolution provides an algorithm for compressing graph signals. Denoising problems have long been studied in image processing. Many models have been proposed for image restoration \cite{milanfar2013tour}. In particular, wavelets provide a sparse and multi-scale representation for images and have proved an impressive regularizer for reducing Gaussian white noise for signals in 2D \cite{figueiredo2003algorithm,cai2012image,Dong2013mra,Shen2010wavelet}. Graph spectral theory and graph wavelets have been widely used for image processing \cite{cheung2018graph}. Our convolution uses graph framelets and provides a solution to the denoising model for structured data using GNN training.

\section{Multiresolution Analysis of Graph Framelets} \label{sec:mra}
\begin{figure*}[h]
    \centering
    \begin{annotate}{\includegraphics[width=0.88\textwidth]{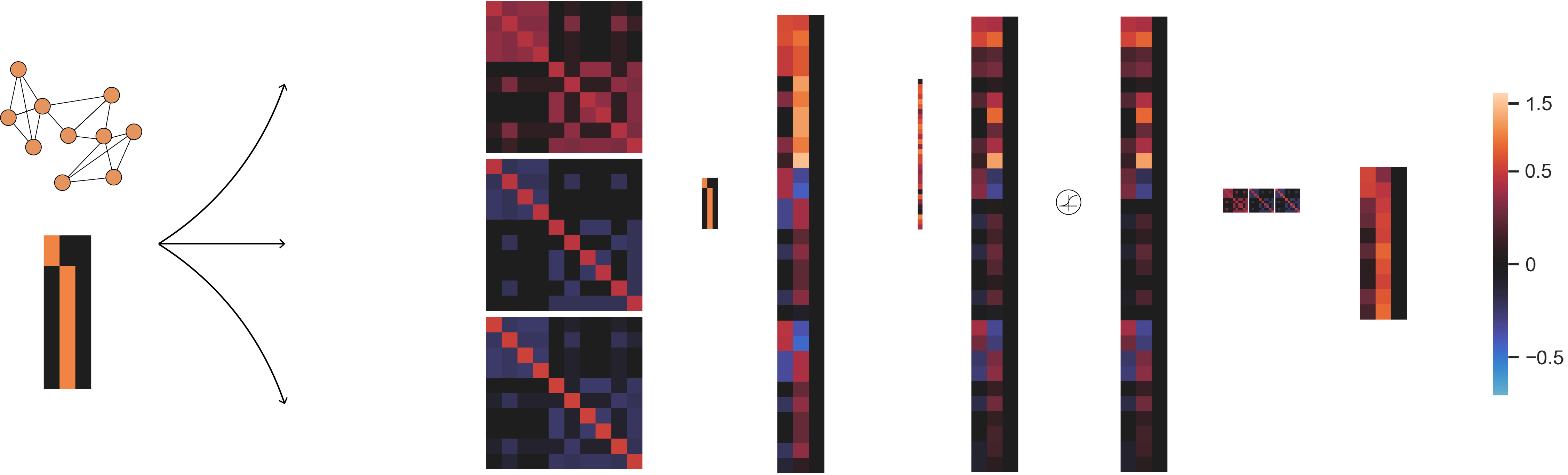}}{1}
    \arrow{-1.2,0}{-0.2,0}
    \arrow{0.65,0}{1.65,0}
    \arrow{2.3,0}{3.1,0}
    \arrow{3.9,0}{5.2,0}
    \note{-6.9,-1.7}{Feature Matrix}
    \note{-3.8,1.5}{Low-Pass}
    \note{-3.8,0.1}{High-Pass}
    \note{-3.8,-0.2}{(level 1)}
    \note{-3.8,-1.3}{High-Pass}
    \note{-3.8,-1.6}{(level 2)}
    \note{1.05,0.4}{${\rm diag}(\;\;)$}
    \draw [decorate,decoration={brace,amplitude=5pt,mirror,raise=4ex}]
    (-6,-1.7) -- (0,-1.7) node[midway,yshift=-2.8em]{{\scriptsize Decomposition}};
    \draw [decorate,decoration={brace,amplitude=5pt,mirror,raise=4ex}]
    (3.6,-1.7) -- (5.7,-1.7) node[midway,yshift=-2.8em]{{\scriptsize Reconstruction}};
    \end{annotate}
    \vspace{-3mm}
    \caption{Computational flow of the proposed undecimated framelet graph convolution (UFGConv). This is an illustration with shrinkage activation, which will be discussed in detail in Section~\ref{sec:framelet convolution}. Given a graph with structure (adjacency matrix) and feature information, the target is to properly embed the graph by graph convolution. The demonstrative sample is a graph with $10$ nodes and $3$ features extracted from \textbf{PROTEINS} in \textbf{TUDataset}. The framelet dilation and scale level are both set to default value $2$. The \textsc{UFGConv} applies tensor-based framelet transform, and constructs one low-pass and two high-pass \emph{framelet transform matrices} $\analOp_{r,j}$, which are then multiplied by the input feature matrix to produce the framelet coefficients. Moreover, these coefficients are processed by the trainable network filter and compressed by the shrinkage. Finally, the activated coefficients are reconstructed and sent back to the spatial domain as the convolution output by using the framelet transform matrices again with transposed alignment.}
    \label{fig:tgft}
    \vspace{-4mm}
\end{figure*}

Our convolution uses the undecimated graph framelets and their transforms, which were introduced by \citet[Section~3]{dong2017sparse}; \citet[Section~4.1]{zheng2020decimated}. Framelets on a specific graph $\gph=(V,\EG,\wG)\in L_2(\gph)$ are defined by a \emph{filter bank} and the spectrum of its graph Laplacian $\gl$. The filter bank $\filtbk:=\{\maska;\maskb[1],\ldots,\maskb[n]\}\in l_0(\Z)$ for a framelet system is a set of compactly supported sequences, where $n$ denotes the number of high pass filters. The \emph{low-pass} and \emph{high-pass filters} of the framelet transforms, $\maska$ and $\maskb[r]$, distill and represent the approximation and detail information of the graph signal. The scaling functions $\Psi=\{\scala;\scalb^{(1)},\ldots,\scalb^{(n)}\}$ with respect to the filter bank $\filtbk$ are used to generate the framelets. The $\scala$, $\scalb^{(r)}$ and their Fourier transforms are in $L_2(\R)$. For $\xi\in\R$, the filters satisfy the classic refining equations
\begin{equation*}
\label{eq:refinement:nonstationary}
    \FT{\scala}(2\xi) = \FS{\maska}(\xi)\FT{\scala}(\xi),\;
    \FT{\scalb^{(r)}}(2\xi) = \FS{\maskb}(\xi)\FT{\scala}(\xi), \; r=1,\ldots,n.
\end{equation*}

\paragraph{Graph Framelets}
Suppose $\{(\eigvm,\eigfm)\}_{j=1}^{N}$ are the eigenvalue and eigenvector pairs for $\gL$ of graph $\gph$ with $N$ nodes. The (undecimated) framelets at \emph{scale level} $j=1,\ldots,J$ for graph $\gph$ with the above scaling functions are defined, for $n = 1,\ldots,r$, by
\begin{equation}\label{eq:ufr}
\begin{aligned}
    \cfra(v) &=  \sum_{\ell=1}^{\NV} \FT{\scala}\left(\frac{\eigvm}{2^{j}}\right)\conj{\eigfm(\uG)}\eigfm(v)\\
    \cfrb{n}(v) &=  \sum_{\ell=1}^{\NV} \FT{\scalb^{(n)}}\left(\frac{\eigvm}{2^{j}}\right)\conj{\eigfm(\uG)}\eigfm(v),
\end{aligned}
\end{equation}
where $\cfra$ or $\cfrb{r}$ is the low-pass or high-pass framelet translated at node $p$. The low-pass and high-pass \emph{framelet coefficients} for a signal $f$ on graph $\gph$ are $v_{j,p}$ and $w_{j,p}^r$, which are the projections $\InnerL{\cfra,f}$ and $\InnerL{\cfrb{r},f}$ of the graph signal onto framelets at scale $j$ and node $p$. Here we use Haar-type filters for framelets \cite{dong2017sparse}. The dilation factor is $2^j$ with the \emph{dilation} (base) $2$.

The above framelet system is \emph{tight} if it provides an exact representation for any function in $L_2(\gph)$. The tightness is determined by filter banks (See Theorem 1 in Appendix B). This condition guarantees a unique representation of a graph signal with framelet coefficients. It also helps manipulate the graph data in the framelet (frequency) domain. 

The graph convolution developed in this work with a tight graph framelet system is effective for several reasons. First, the undecimated framelet transforms have a simplified formula that uses only graph Laplacian and filtered Chebyshev polynomial approximation, as will be further explained below. The resulting framelet transforms can be written as a sparse tensor and the corresponding framelet coefficients are evaluated fast in tensor computation. See Figure~\ref{fig:tgft} for the computational flow of framelet convolution. Moreover, the tight framelets on the graph have a low redundancy rate, which is analogous to the framelets on a manifold as considered by  \citet{WaZh2018,wang2019tight}. We give more discussion about framelets in the Appendix. 

\paragraph{Tensor-based $\gph$-framelet Transforms}
Framelet transforms map between a graph signal $f$ and its representative framelet coefficients. We point out an approximate formula for \eqref{eq:ufr}, which exploits the Chebyshev polynomial approximation to the filters $\FT{\scala}$ and $\FT{\scalb^{(r)}}$, and thus gives a simplified version and fast evaluation for framelet transforms. By the framelet transform theorem (in Appendix), the corresponding framelet transforms are implemented by $\analOp$ and $\synOp$, the decomposition and reconstruction operators. 
For graph signal $f$, we define $\analOp=\{\analOp_{r,j}| r=1,\dots,n; j=1,\dots,J\}\cup\{\analOp_{0,J}\}$ entry-wisely, where $\analOp_{r,1}f = \mathcal{T}_{r}^k(2^{-J}\gl)f$, and 
\begin{equation*}
    \analOp_{r,j}f = \mathcal{T}_{r}^k(2^{K+j-1}\gl)\mathcal{T}_{0}^k(2^{K+j-2}\gl)\cdots \mathcal{T}_{0}^k(2^{-K}\gl)f
\end{equation*}
for $j=2,\dots,J$. Here $\mathcal{T}_{r}^k$ is the $r$-degree Chebyshev polynomial, $\gl$ is the graph Laplacian, and $K$ is the constant determined by the maximum eigenvalue of $\gl$ that satisfies $\eigvm[{\rm max}]\leq 2^K\pi$. The $\analOp_{0,J} f=\{v_{J,p}\}_{p\in V}$ are the \emph{low-pass coefficients} and $\analOp_{r,j} f=\{w_{j,p}^r\}_{p\in V}$ are \emph{high-pass coefficients} of $f$. The $j$ indicates the scale level, and $r=1,\dots,n$ with $n$ the number of high-passes. The reconstruction operator $\synOp$ is the realignment of the framelet transform matrices of the decomposition operator $\analOp$.

Figure~\ref{fig:tgft} gives the fast $\gph$-framelet transform algorithm based on tensorized $\analOp$ and $\synOp$ with scale level $2$ and shrinkage threshold $\sigma=1$ (the meaning of $\sigma$ will be discussed in detail in Section~\ref{sec:shrinkiage and signal compression}). In practice, we turn the computation into merely sparse matrix multiplication by properly aligning the low-pass and high-pass elements of $\analOp$ and $\synOp$. The tensor-based framelet transforms have time complexity $\bigo{}{N^2(nJ+1)Kd}$ and space complexity $\bigo{}{N^2(nJ+1)d}$ for an $N$-node graph and $d$ features. Here, the $n$, $J$ and $K$ are constants independent of graph data. See Appendix for an empirical study of the complexity of framelet transforms on benchmarks.

\definecolor{blue1}{RGB}{85,113,171}
\definecolor{orange1}{RGB}{209,136,92}
\definecolor{green1}{RGB}{106,166,110}
\begin{figure*}[t]
    \centering
    \scriptsize
    \begin{annotate}{\includegraphics[width=0.96\textwidth]{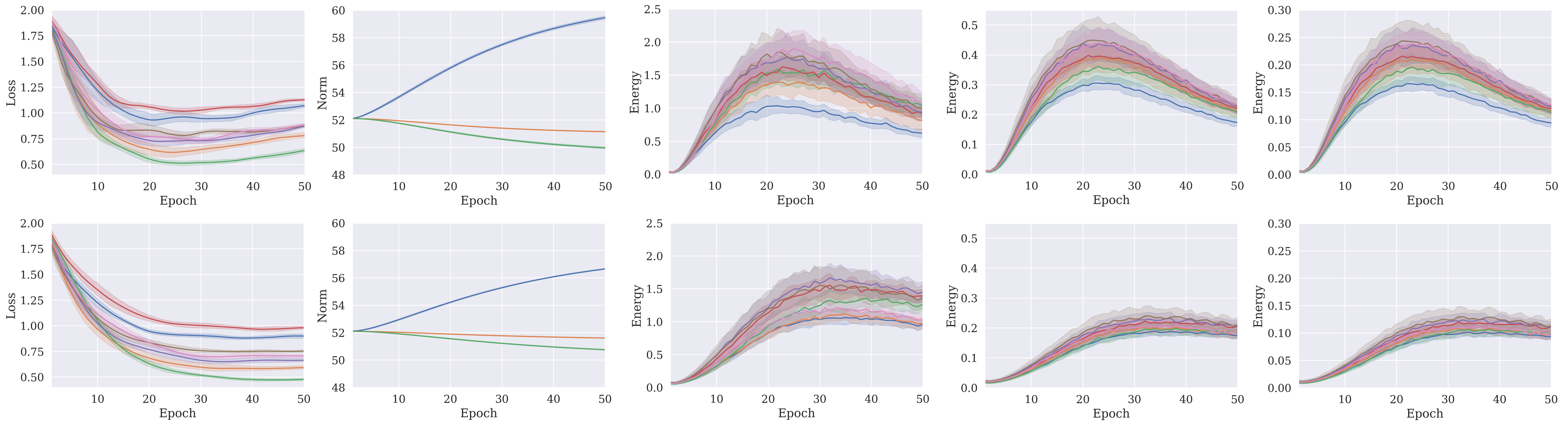}}{1}
    \note{-6.5,2.3}{Loss}
    \note{-3.3,2.3}{$\|\btheta\|$}
    \note{0.1, 2.3}{Low-Pass}
    \note{3.4, 2.3}{High-Pass 1}
    \note{6.7, 2.3}{High-Pass 2}
    \note{-2.3, 1.8}{\tiny\textcolor{blue1}{Low-Pass}}
    \note{-2.4, 1}{\tiny\textcolor{orange1}{High-Pass 1}}
    \note{-2.4, 0.6}{\tiny\textcolor{green1}{High-Pass 2}}
    \note{-2.3, -0.45}{\tiny\textcolor{blue1}{Low-Pass}}
    \note{-2.4, -1.2}{\tiny\textcolor{orange1}{High-Pass 1}}
    \note{-2.4, -1.55}{\tiny\textcolor{green1}{High-Pass 2}}
    \draw[ ] (-8.4,1.2,0) node [rotate=90] {ReLU};
    \draw[ ] (-8.4,-0.9,0) node [rotate=90] {Shrinkage};
    \end{annotate}
    \vspace{-4mm}
    \caption{Learning behavior of the final framelet convolutional layer of GNN with two \textsc{UFGConv} for \textbf{Cora}. Top is for \textsc{UFGConv} with ReLU activation in \eqref{eq:fgtconv relu}. Bottom is for \textsc{UFGConv} with Shrinkage activation in \eqref{eq:fgtconv shrinkage}. From left to right we show some key network learning properties: training loss, $l_2$ norm of network filter $\btheta$, power spectrum of framelet coefficients for low-pass and high-passes at scale levels $1$ and $2$. The curves of the quantities for each of $7$ label classes are shown. Framelets provide a good feature representation and shrinkage makes training more stable, where central information energy is conserved via thresholding high-pass coefficients.}
    \label{fig:loss energy norm}
    \vspace{-3mm}
\end{figure*}

\section{Framelet Convolution}\label{sec:framelet convolution} 
With the above $\gph$-framelet transforms, we can define a \emph{framelet (graph) convolution} similar to the spectral graph convolution. For \emph{network filter} $\btheta$ and input \emph{graph feature} $X\in \Rd[N\times d]$ of the graph $\gph$ with $N$ nodes, we define
\begin{equation}\label{eq:fgtconv relu}
    \btheta\star X = {\rm ReLU}\left(\synOp\left({\rm diag}(\btheta) (\analOp X')\right)\right),\; X'=X W.
\end{equation}
As mentioned, $\analOp X'$ is the framelet coefficient matrix for the transformed $X'$, where $\analOp$ is a sequence of $nJ+1$ transform matrices (each of size $N\times N$) for low-pass and high-passes. The size of the vector $\btheta$ is $(n J +1)N$ which matches the total number of the framelet coefficients for each feature. The network filter $\btheta$ lies in the frequency domain, each component of which is multiplied to the corresponding row of $\analOp X'$. The matrix $W$ in \eqref{eq:fgtconv relu} is a trainable weight matrix with dimension $d\times d'$.

We can replace the ReLU activation in \eqref{eq:fgtconv relu} with a wavelet shrinkage threshold function, or \emph{shrinkage}. As the framelet coefficients lie in multiple scales, the ``one size fits all'' criterion of ReLU could potentially damage the multi-scale property of the framelet representation. In contrast, the shrinkage threshold adapts to the varying scales of the coefficients and provides a more precise cutoff. Thus we can use fewer coefficients while maintaining a comparable performance in framelet representation, see Section~\ref{sec:experiment ufgconv}. The framelet convolution with \emph{shrinkage activation} is
\begin{equation}\label{eq:fgtconv shrinkage}
     \synOp\left({\rm Shrinkage}\bigl({\rm diag}(\btheta) (\analOp X')\bigr)\right),\; X'=X W.
\end{equation}
Different from ReLU which works in the spatial domain, in \eqref{eq:fgtconv shrinkage}, the shrinkage activation is carried on before the framelet reconstruction $\synOp$ as the shrinkage thresholds for the high-pass coefficients in the framelet domain.
Figure~\ref{fig:tgft} demonstrates a computational flow of the framelet convolution that learns the graph embedding of a graph with feature matrix $X^{\rm in}\in\R^{N\times d}$ by a framelet system of $2$ scale levels $(j=1,2)$ and $1$ high-pass filter $(r=0,1)$. Here, we omit the feature transformation step for a simple illustration. The decomposition $\analOp=\{\analOp_{0,2};\analOp_{1,1},\analOp_{1,2}\}$ transforms the input feature matrix with one low-pass and two high-pass operators. The operators can be rearranged to $\analOp^{\natural}:=[\analOp_{0,2},\analOp_{1,1},\analOp_{1,2}]^{\top}$ of three $N\times N$ sparse matrices. The coefficients $\analOp^{\natural}X^{\rm in}\in \R^{3N\times d}$ can then be obtained by matrix multiplication. The network learning propagates in the frequency domain, where one applies the network filter $\btheta$ to the framelet coefficients $\analOp^{\natural}X^{\rm in}$. For the framelet convolution with shrinkage, the filtered coefficients would be activated before applying $\synOp^{\natural}=(\analOp^{\natural})^{\top}$. Otherwise the ReLU activation will act on the reconstructed signal in the spatial domain.

\section{Shrinkage and Signal Compression} \label{sec:shrinkiage and signal compression}
Wavelet shrinkage is intimately linked to multiresolution properties of the wavelet transform in the classic wavelet theory. The shrinkage only applies to finer scales, i.e., detail coefficients \cite{donoho1995wavelet}, so that the wavelet scalogram (a paradigm of the time-frequency energy localization of a signal) experiences a minimal change. This property promises a meaningful estimator for signal compression and an explainable graph convolution with framelets.

\paragraph{Sparsity and Compression}
The high-pass coefficients in the frequency domain can be cut off by shrinkage thresholding. For example, the \emph{soft-thresholding} \cite{donoho1994ideal,donoho1995wavelet,tibshirani1996regression} defines
\begin{equation*}
    {\rm Shrinkage}(x) = {\rm sgn}(x)(|x|-\lambda)_{+} \quad\forall x\in \R,
\end{equation*}
where $\lambda$ is the threshold value. Any $x$ with its absolute value less than $\lambda$ shall return to zero. Applying the above soft-thresholding to the shrinkage activation in \eqref{eq:fgtconv shrinkage} only influences small high-pass framelet coefficients. We also consider the scale-dependent selection threshold with demarcation point \cite{donoho1995noising}: $\lambda = \sigma\sqrt{2\log(N)}/\sqrt{N}$ for $N$ coefficients. The hyperparameter $\sigma$ is an analogue to the noise level of the wavelet denoising model. We let $\sigma$ be associated with the magnitude order of the coefficients so it reflects the scale of the framelet representation. 

The shrinkage benefits framelet graph convolution by reducing noise in framelet coefficients and compressing the signal in the frequency domain simultaneously. The traditional wavelet denoising for 1D functions suggests that using shrinkage at high-pass coefficients can effectively filter out the Gaussian white noise in the mean square error sense. This is also true for our case when we embed shrinkage in graph convolution. In Figure~\ref{fig:attack}, the framelet convolution with shrinkage activation (\textsc{UFGConv-S}) outperforms the ReLU case (\textsc{UFGConv-R}) for reducing node and structure noises. Both methods surpass the classic spatial convolutions \textsc{GCN} \cite{KiWe2017} and \textsc{GAT} \cite{velivckovic2017graph}. Apart from denoising, graph convolution with shrinkage diminishes the proportion of non-zero framelet coefficients while maintains a comparable learning performance. We define the \emph{compression ratio} for a shrinkage framelet convolutional layer as the ratio of the number of non-zero coefficients after and before shrinkage. Tables~\ref{tab:citation_node} and \ref{tab:ogbn_arxiv} show that \textsc{UFGConv-S} compresses up to 70\% non-zero coefficients with top performance for various node classification tasks.

\paragraph{Framelet Spectrum, Training Loss and Network Capacity}
The coefficients after shrinkage activation are proportional to the framelet power spectrum at the coefficient scale. We thus let the threshold level $\sigma$ proportionate upon the framelet energy $\|\analOp_{r,j}X\|^2$ ($r>0$) for high-passes. For example, the framelet spectrum curves in training for the \textsc{UFGConv} of Figure~\ref{fig:loss energy norm} show a higher magnitude order of the low-pass (column 3) than those of high-passes (columns 4-5). This is because coefficients in high-passes reflect more detailed characteristics than in low-pass. Compared with the ReLU case (row 1), the shrinkage activation (row 2) filters out some high-pass coefficients in graph convolution, which results in much smaller framelet spectra for high-passes. In contrast, the low-pass shrinkage involves no cutoff, and the energy is less distinguishable from the ReLU case.

The training loss curve of each output feature (column 1) indicates that shrinkage allows for more stable training, with a monotonically decreasing loss. The splitting in low-pass and high-passes for loss suggests a more flexible and precise control of the training. It also opens the possibility of designing a weighted new loss taking account of framelet spectrum. Moreover, the $l_2$ norm of the network filter $\btheta$ (column 2) has an increasing trend for the low-pass part and a decreasing trend for the high-pass parts during training. This observation is identical to the spectral bias \cite{rahaman2019spectral}, whereby the fitting capacity of the framelet convolutional GNN with either ReLU or shrinkage activation comes from the low-pass channel.

\section{Robustness of Framelet Convolution under Feature and Structure Noises} \label{sec:robustness}
\begin{figure}[t]
    \centering
    \includegraphics[width=\columnwidth]{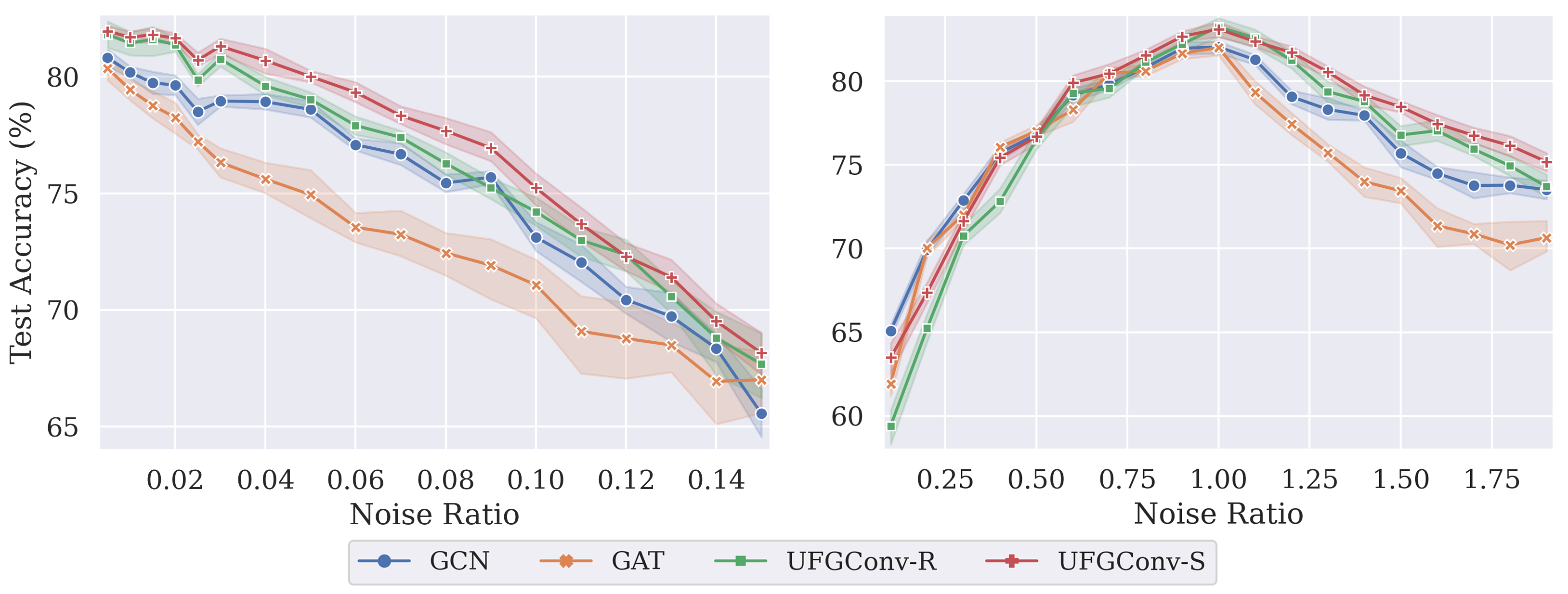}
    \vspace{-6mm}
    \caption{Node (left) and structure (right) perturbation analysis on \textbf{Cora}. Results from the other two datasets are in Appendix. The framelet dilation and scale level are both set to default value $2$, and the optimal threshold $\sigma$ for shrinkage is searched from $\{0.05, 0.1, 0.15\}$. Framelet convolution with shrinkage performs the best under both node and structure perturbations.} \label{fig:attack}
    \vspace{-4mm}
\end{figure}

Real-world data are usually noisy. For instance, graph data are sometimes polluted due to adversarial attacks as GNNs exchange node information \cite{xu2020adversarial}, where node feature and graph structure could both be perturbed. 

Our shrinkage framelet convolution has a motivation from the image deconvolution (or image restoration) model. Given the original and observed (degraded) image features $x$ and $y$, one defines
$y = Hx + \epsilon$, 
where $H$ represents the convolution matrix of $x$. The noise $\epsilon$ is assumed multivariate Gaussian. In statistical formulation, the model solves
\begin{equation*}
    \widehat{x} = {\rm arg~min}_{x} \bigl\{\log(\Pr(y|x)) - {\rm Pen}(x)\bigr\}
\end{equation*}
with a penalty function ${\rm Pen}$. On the real axis, wavelets are critical in restoring $x$ from noisy $y$ \cite{figueiredo2003algorithm,cai2012image,Dong2013mra,Shen2010wavelet}. With wavelet transform $\Phi$, the maximum penalized likelihood estimator (MPLE) takes the form
\begin{equation}\label{eq:deconv}
    \widehat{x} = \Phi^{\top} D \Phi y,  
\end{equation}
where $D$ acts as a denoising operation of $\Phi y$.

The above wavelet-based convolution can be generalized to graph data restoration with $x$ and $y$ replaced by clean and distorted features on the graph. That is, 
\begin{equation}\label{eq:adversarial attack model}
    y = Px + \epsilon,
\end{equation}
where $\epsilon$ is the entry-wise noise. The linear transform $P$ is a permutation of node features or a change of the adjacency matrix. Similar models using graph Fourier spectrum were considered in image restoration, see \citet[Sect~V.B]{cheung2018graph} and \citet{milanfar2013tour}. 

Similar to \eqref{eq:deconv}, our graph convolution replaces $\Phi$ with the graph framelet transform $\analOp$, and $D$ with the trainable network filter $\btheta$. Our shrinkage activation is partly motivated by LASSO \cite{tibshirani1996regression}. The latter uses shrinkage thresholding to denoise the signal in \eqref{eq:adversarial attack model} in high dimensions. 
Based on this connection, our proposed framelet convolution in Section~\ref{sec:framelet convolution} has a good potential against perturbation from \eqref{eq:adversarial attack model}. We test \textsc{UFGConv} with perturbed nodes and edges in citation network datasets (\textbf{Cora}, \textbf{Citeseer} and \textbf{Pubmed}). The noise $\epsilon$ follows Bernoulli distribution to match binary node features of the tasks. That is, we randomly change node feature or edge weight from 0/1 to its opposite 1/0. Figure~\ref{fig:attack} reports the test accuracy of \textsc{UFGConv}, \textsc{GCN} and \textsc{GAT}. The $x$-axis indicates the proportion of the distorted nodes or edges (which is equivalent to the signal-to-noise ratio (SNR)). As the SNR increases, \textsc{UFGConv} behaves well ahead of \textsc{GCN} and \textsc{GAT} while with higher test accuracy and smaller variance. The slightly lower performance of \textsc{UFGConv} only occurs when more than 50\% edges are removed, where misinformation dominates the graph structure. Moreover, the shrinkage thresholding in all situations manages to filter out more noise and thus achieves higher performance. This example illustrates the effectiveness of framelet convolution in predicting node property with node features or structure of graphs that are distorted.

\section{Framelet Pooling}
Graph pooling is a critical ingredient of GNNs when the model is predicting graph-level properties with a constant feature dimension but varying graph size and structure. We propose \emph{framelet pooling} for GNNs using framelet transforms. As an illustrative example, we use $2$ scale levels for framelet decomposition. Similar to the graph convolution in Section~\ref{sec:framelet convolution}, given a graph with feature matrix $X\in \R^{N\times d}$, we can obtain a set of framelet coefficients $\analOp_{r,j}f$ including one low pass $\analOp_{0,2}f$ at level $2$ and two high passes $\analOp_{1,1}f$ and $\analOp_{1,2}f$ at levels $1$ and $2$, respectively. Each scale-wise framelet coefficient is an $N\times d$ real-valued matrix, and its $i$th feature column $(\analOp_{r,j}X)_i$ for $i=1,\dots,d$ would be aggregated by the sum, or the sum of squares of the elements. The two aggregation methods correspond to two framelet pooling strategies (see below). The calculation compresses the $N\times d$ coefficients to a $d$-dimensional vector, and the pooled output from the three framelet coefficients results in $3$ $d$-dimensional vectors. Figure~\ref{fig:ufgpool} visualizes the computational flowchart for our pooling model.

\begin{figure}[t]
    \centering
    \begin{annotate}{\includegraphics[width=0.95\columnwidth]{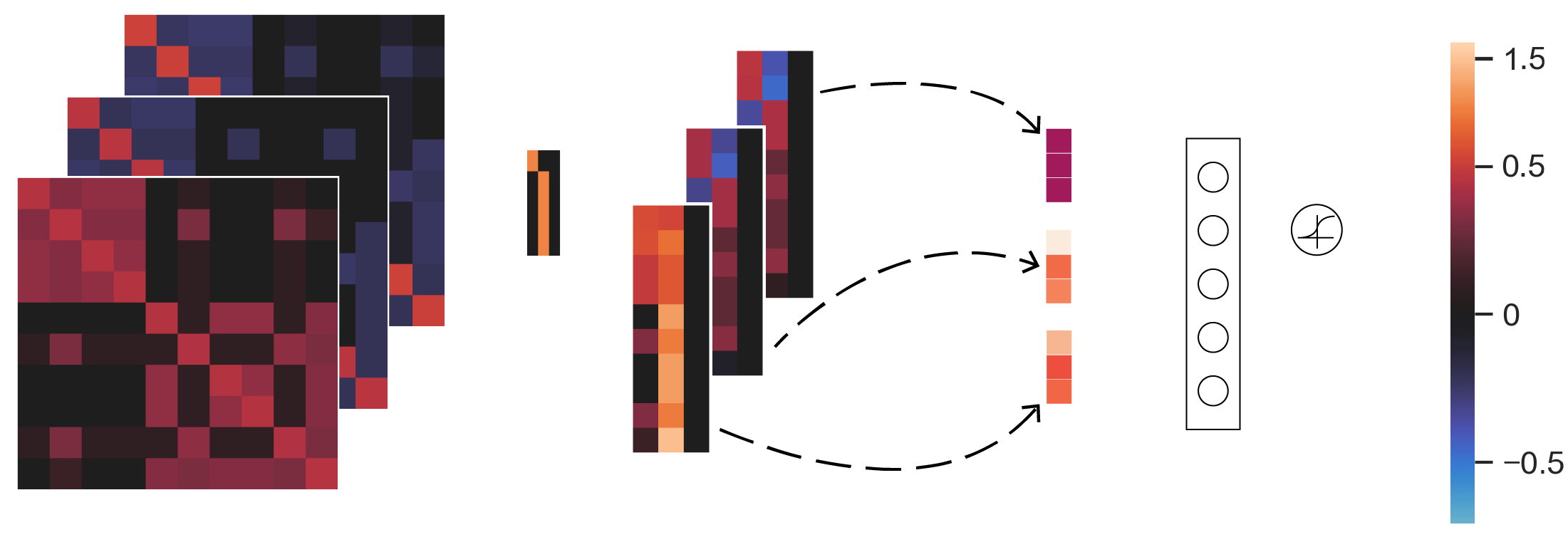}}{1}
    \arrow{-1.6,0}{-0.9,0}
    \arrow{1.5,0}{1.9,0}
    \arrow{2.4,0}{2.9,0}
    \note{-2.7,-1.25}{Framelet Transform Matrices}
    \note{2.2,-1}{Classifier}
    \note{3.1,0}{$\hat{y}$}
    \end{annotate}
    \vspace{-7mm}
    \caption{Framelet pooling for graph property prediction. The three framelet transform matrices are retrieved from Figure~\ref{fig:tgft} with the same protein sample and parameter setting. The scale-wise framelet coefficients are aggregated to three vectors by sum or sum of squares (framelet spectrum). The (1 low pass and 2 high passes) vectors are then concatenated as the readout for the classifier.}
    \label{fig:ufgpool}\vspace{-6mm}
\end{figure}

The framelet pooling benefits the network training by employing the information from multi-scales, as all scales in the framelet representation of the graph signal is taken into account. Depending on how we aggregate the framelet coefficients, we distinguish the two strategies by \textsc{UFGPool-SUM} or \textsc{UFGPool-Spectrum}. The latter aggregates the nodes by the wavelet (power) spectrum (that is, the sum of absolute squares of framelet coefficients over nodes, $\sum_{p\in V}|v_{j,p}|^2$ and $\sum_{p\in V}|w_{j,p}^{r}|^2$). In this way, the total information of the graph signal $X^{\rm in}$ is well-conserved after the pooling. The sum of wavelet power spectrum is equal to the total energy of the signal, that is, $\|X^{\rm pooled}\|=\|X^{\rm in}\|$ (see Theorem 1(iii) in Appendix). We present empirical evidence of the precedence of the proposed framelet pooling over existing graph pooling methods in Section~\ref{sec:experiment ufgpool}.

\section{Experiments}\label{sec:exp}
In this section, we show a variety of numerical tests for our framelet convolution and pooling. Section~\ref{sec:experiment ufgconv} tests the performance of framelet convolution (\textsc{UFGConv}) on node classification benchmarks. Section~\ref{sec:experiment ufgpool} studies the ablation for the proposed framelet pooling (\textsc{UFGPool}). Section~\ref{sec:sensitity} provides a sensitivity analysis for \textsc{UFGConv} in terms of dilation and scale. All experiments run in PyTorch on NVIDIA\textsuperscript{\textregistered} Tesla V100 GPU with 5,120 CUDA cores and 16GB HBM2 mounted on an HPC cluster.

\subsection{Framelet Convolution for Node Classification}
\label{sec:experiment ufgconv}
We test \textsc{UFGConv} with both ReLU and shrinkage activations on four node classification datasets. We denote the two variants by \textsc{UFGConv-R} and \textsc{UFGConv-S}. 

\paragraph{Dataset} 
The first experiment of node classification tasks is conducted on \textbf{Cora, Citeseer} and \textbf{Pubmed}, which are three benchmark citation networks. Moreover, we employ \textbf{ogbn-arxiv} from open graph benchmark \textbf{OGB} \cite{hu2020open} to illustrate the power of our framelet convolution on large-scale graph-structured data. 

\paragraph{Setup} 
We design our \textsc{UFGConv} model with two convolution layers for learning the graph embedding, the output of which is proceeded by a softmax activation for final prediction. Most hyperparameters are set to default, except for learning rate, weight decay, hidden units and dropout ratio in training. A grid search is conducted for fine tuning on these hyperparameters from the search space detailed in Appendix. Both methods are trained with the \textsc{Adam} optimizer. The maximum number of epochs is $200$ for citation networks and $500$ for \textbf{ogbn-arxiv}. All the datasets follow the standard public split and processing rules. The average test accuracy and its standard deviation come from $10$ runs. 

\paragraph{Baseline} 
The \textsc{UFGConv-R} and \textsc{UFGConv-S} are compared against other methods for node classification tasks. We consider multiple baseline models that are applicable to the tasks. For citation networks, the reported accuracy are retrieved from public results:
\textsc{MLP}, \textsc{DeepWalk} \citep{perozzi2014deepwalk}, \textsc{Chebyshev} \citep{DeBrVa2016} and \textsc{GCN} \citep{KiWe2017} from \citet{KiWe2017}; \textsc{Spectral CNN} \citep{bruna2013spectral} and \textsc{GWNN} \citep{xu2019graph}; \textsc{MPNN} \cite{Gilmer_etal2017}, \textsc{GraphSAGE} \cite{hamilton2017inductive}, \textsc{LanczosNet} \cite{liao2018lanczosnet} and \textsc{DCNN} \cite{singer2009diffusion} from \citet{liao2018lanczosnet}; and \textsc{GAT} \citep{velivckovic2017graph} from their authors. For \textbf{ogbn-arxiv}, we also compared with \textsc{Node2vec} \cite{grover2016node2vec}, \textsc{GraphZoom} \cite{Deng2020GraphZoom}, \textsc{P\&L+C\&S} \cite{huang2020combining}, \textsc{DeeperGCN} \cite{li2020deepergcn}, \textsc{SIGN} \cite{rossi2020sign} and \textsc{GaAN} \cite{zhang2018gaan} from the \textbf{OGB} leaderboard.

\begin{table}[t]
\caption{Test accuracy (in percentage) for citation networks with standard deviation after $\pm$. Compression ratio in \textbf{\textcolor{green!40!black}{(green)}} is the ratio of numbers of nonzero coefficients after and before shrinkage, and is with threshold level $\sigma=1$.}
\label{tab:citation_node}
\footnotesize
\begin{center}
\begin{tabular}{lccc}
\toprule
\textbf{Method} & \textbf{Cora} & \textbf{Citeseer} & \textbf{Pubmed}\\
\midrule
\textsc{MLP} & $55.1$ & $46.5$ & $71.4$\\
\textsc{DeepWalk} & $67.2$ & $43.2$ & $65.3$\\
\textsc{Spectral} & $73.3$ & $58.9$ & $73.9$\\
\textsc{Chebyshev} & $81.2$ & $69.8$ & $74.4$\\
\textsc{GCN} & $81.5$ & $70.3$ & $79.0$\\
\textsc{GWNN} & $82.8$ & \textcolor{black}{$\boldsymbol{71.7}$} & \textcolor{black}{$\boldsymbol{79.1}$}\\
\textsc{GAT} & \textcolor{black}{$\boldsymbol{83.0}${\scriptsize$\pm0.7$}} & \textcolor{violet}{$\boldsymbol{72.5}${\scriptsize$\pm0.7$}} & $79.0${\scriptsize $\pm0.3$}\\
\textsc{MPNN} & $78.0${\scriptsize$\pm1.1$} & $64.0${\scriptsize$\pm1.9$} & $75.6${\scriptsize$\pm1.0$}\\
\textsc{GraphSAGE} & $74.5${\scriptsize$\pm0.8$} & $67.2${\scriptsize$\pm1.0$} & $76.8${\scriptsize$\pm0.6$}\\
\textsc{LanczosNet} & $79.5${\scriptsize$\pm1.8$} & $66.2${\scriptsize$\pm1.9$} & $78.3${\scriptsize$\pm0.3$}\\
\textsc{DCNN} & 79.7{\scriptsize$\pm0.8$} & $69.4${\scriptsize$\pm1.3$} & $76.8${\scriptsize$\pm0.8$}\\
\midrule
\textsc{UFGConv-S} & \textcolor{violet}{$\boldsymbol{83.0}${\scriptsize$\pm0.5$}} & $71.0${\scriptsize$\pm0.6$} & \textcolor{violet}{$\boldsymbol{79.4}${\scriptsize$\pm0.4$}} \\
{\textcolor{green!40!black}{(Compression)}} & \textcolor{green!40!black}{$(\boldsymbol{47.7})$} & \textcolor{green!40!black}{$(\boldsymbol{39.0})$} & \textcolor{green!40!black}{$(\boldsymbol{27.7})$}\\
\textsc{UFGConv-R} & \textcolor{red}{$\boldsymbol{83.6}${\scriptsize$\pm0.6$}} &
\textcolor{red}{$\boldsymbol{72.7}${\scriptsize$\pm0.6$}} &
\textcolor{red}{$\boldsymbol{79.9}${\scriptsize$\pm0.1$}}\\
\bottomrule\\[-2.5mm]
\multicolumn{4}{l}{$\dagger$ The top three are highlighted by \textbf{\textcolor{red}{First}}, \textbf{\textcolor{violet}{Second}}, \textbf{Third}.}
\end{tabular}
\end{center}
\vspace{-6mm}
\end{table}

\begin{table}[t]
\caption{Test accuracy (in percentage) for \textbf{ogbn-arxiv} with standard deviation after $\pm$. The compression ratio for \textsc{UFGConv-S} with shrinkage threshold level $\sigma=1$ is \textcolor{green!40!black}{$\boldsymbol{64.2\%}$}.}
\label{tab:ogbn_arxiv}
\footnotesize
\begin{center}
\begin{tabular}{lccr}
\toprule
\textbf{Method} & \textbf{Test Acc.} & \textbf{Val. Acc.} & \textbf{\#Params}\\
\midrule
\textsc{MLP} & $55.50${\scriptsize $\pm0.23$} & $57.65${\scriptsize $\pm0.12$} & $110,120$\\
\textsc{Node2vec} & $70.07${\scriptsize $\pm0.13$} & $71.29${\scriptsize $\pm0.13$} & $21,818,792$\\
\textsc{GraphZoom} & $71.18${\scriptsize $\pm0.18$} & $72.20${\scriptsize $\pm0.07$} & $8,963,624$\\
\textsc{P\&L + C\&S} & $71.26${\scriptsize $\pm0.01$} & $73.00${\scriptsize $\pm0.01$} & $5,160$\\
\textsc{GraphSAGE} & $71.49${\scriptsize $\pm0.27$} & $72.77${\scriptsize $\pm0.17$} & $218,664$\\
\textsc{GCN} & $71.74${\scriptsize $\pm0.29$} & $73.00${\scriptsize $\pm0.17$} & $142,888$\\
\textsc{DeeperGCN} & $71.92${\scriptsize $\pm0.17$} & $72.62${\scriptsize $\pm0.14$} & $491,176$\\
\textsc{SIGN} & \textcolor{black}{$\boldsymbol{71.95}${\scriptsize $\pm0.11$}} & $73.23${\scriptsize $\pm0.06$} & $3,566,128$\\
\textsc{GaAN} & \textcolor{violet}{$\boldsymbol{71.97}${\scriptsize $\pm0.18$}} & -- & $1,471,506$\\
\midrule
\textsc{UFGConv-S} & $70.04${\scriptsize $\pm0.22$} & $71.04${\scriptsize $\pm0.11$} & $1,633,183$\\
\textsc{UFGConv-R} &\textcolor{red}{$\boldsymbol{71.97}${\scriptsize $\pm0.12$}} & $73.21${\scriptsize $\pm0.05$} & $1,633,183$\\
\bottomrule\\[-2.5mm]
\multicolumn{4}{l}{$\dagger$ The top three are highlighted by \textbf{\textcolor{red}{First}}, \textbf{\textcolor{violet}{Second}}, \textbf{Third}.}\\
\end{tabular}
\end{center}
\vspace{-6mm}
\end{table}

\begin{figure}[t]
  \begin{minipage}{0.48\textwidth}
    \centering
    \includegraphics[width=0.9\linewidth]{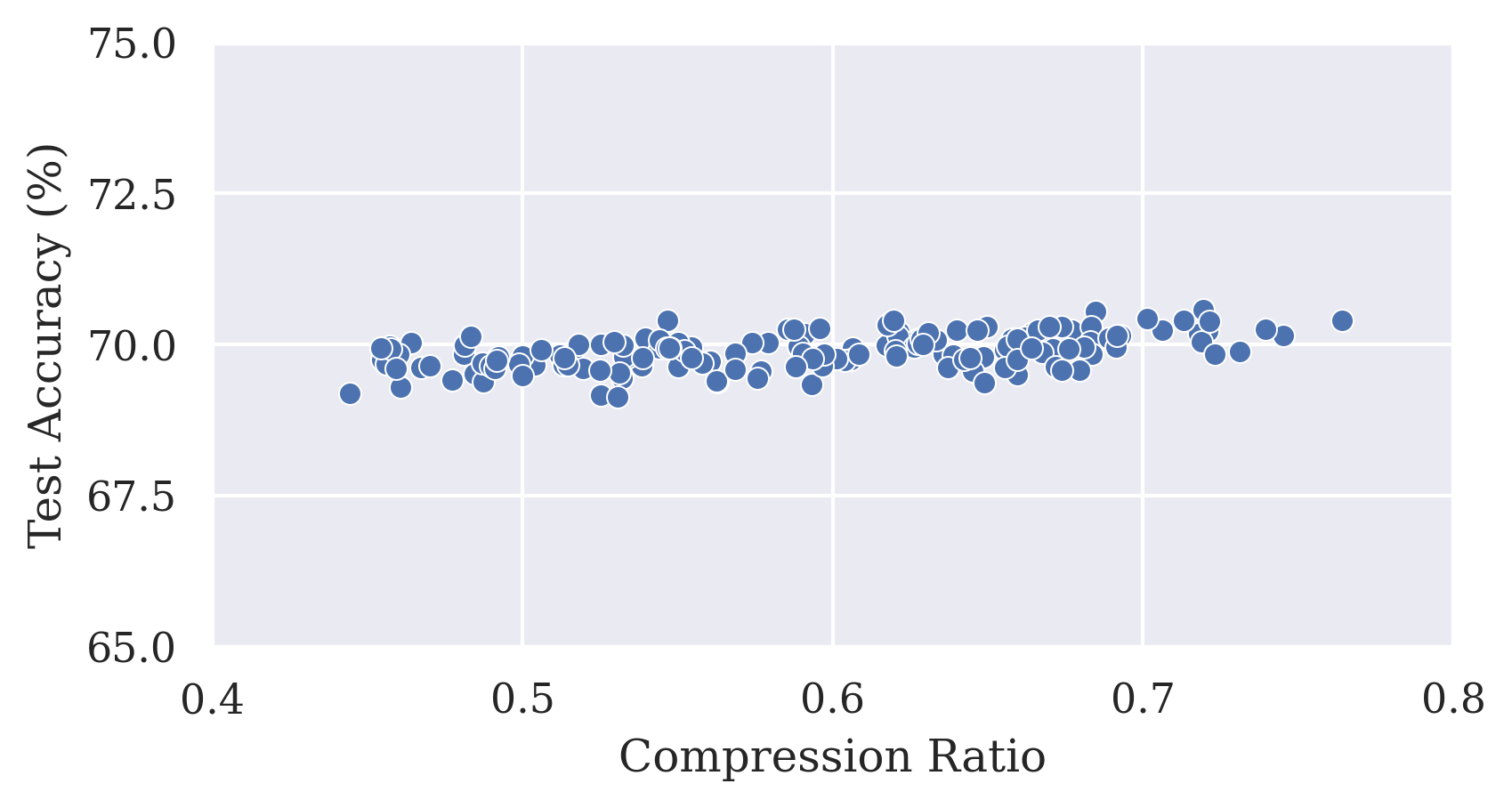}
    \vspace{-4mm}
    \caption{Trade-off between compression ratio and test accuracy in \textsc{UFGConv-S} with \textbf{ogbn-arxiv}.}\label{Fig:tradeoff}
  \end{minipage}\vspace{-4mm}
\end{figure}

\paragraph{Results} 
We report the accuracy score in percentage with the top-3 highlighted in Tables~\ref{tab:citation_node} and~\ref{tab:ogbn_arxiv}. For \textsc{UFGConv-S}, we also report the compression ratio for shrinkage (in green). In Table~\ref{tab:citation_node}, the \textsc{UFGConv-R} method achieves the highest prediction accuracy among all baseline models. The learned \textsc{UFGConv-S} with threshold level $\sigma=1$ trims up to $50\%$ information but still obtains the top-3 rank on two tasks. A similar outstanding performance is reported in Table~\ref{tab:ogbn_arxiv} for the \textbf{ogbn-arxiv} dataset, where the \textsc{UFGConv-R} again ranks first with a moderate number of parameters, and the \textsc{UFGConv-S} with threshold $\sigma=1$ achieves a comparable accuracy at 70\% using 64.2\% information.

We test \textsc{UFGConv-S} with different compression ratios. Ideally, a high test accuracy is preferable to pair with a low compression ratio, and the change in accuracy should be minimal due to the insensitivity of our model to the hyperparameters. However, as shown in Figure~\ref{Fig:tradeoff}, an increasing compression ratio generally results in a slightly higher prediction accuracy as more coefficients for the framelet representation is used by the convolution.

\subsection{Framelet Pooling for Graph Property Prediction}
\label{sec:experiment ufgpool}
The second experiment evaluates two framelet pooling methods, \textsc{UFGPool-Sum} and \textsc{UFGPool-Spectrum}, by ablation studies on graph classification and regression tasks. 

\begin{table*}
\begin{minipage}{0.95\textwidth}
\caption{Performance comparison for graph property prediction. 
\textbf{QM7} is a regression task in MSE; \textbf{ogbg-molhiv} is a classification task in ROC-AUC in percentage; others are for classification in test accuracy in percentage. The value after $\pm$ is standard deviation.}
\label{table:poolingAblation}
\end{minipage}
\footnotesize
\centering
\begin{tabular}{lcccccc}
\toprule
\textbf{Datasets} & \textbf{PROTEINS} & \textbf{Mutagenicity} & \textbf{D\&D} & \textbf{NCI1} & \textbf{ogbg-molhiv} & \textbf{QM7} \\ 
\midrule
\textsc{TopKPool}    & $73.48${\scriptsize $\pm3.57$} & $79.84${\scriptsize $\pm2.46$} & $74.87${\scriptsize $\pm4.12$} & $75.11${\scriptsize $\pm3.45$} & $78.14${\scriptsize $\pm0.62$} & $175.41${\scriptsize $\pm3.16$} \\
\textsc{Attention}   & $73.93${\scriptsize $\pm5.37$} & $80.25${\scriptsize $\pm2.22$} & $77.48${\scriptsize $\pm2.65$} & $74.04${\scriptsize $\pm1.27$} & $74.44${\scriptsize $\pm2.12$} & $177.99${\scriptsize $\pm2.22$} \\
\textsc{SAGPool}     & \textcolor{black}{$\boldsymbol{75.89}${\scriptsize$\pm2.91$}} & $79.86${\scriptsize $\pm2.36$} & $74.96${\scriptsize $\pm3.60$} & $76.30${\scriptsize $\pm1.53$} & $75.26${\scriptsize $\pm2.29$} & \textcolor{black}{$\boldsymbol{41.93}${\scriptsize$\pm1.14$}} \\
\textsc{SUM}         & $74.91${\scriptsize $\pm4.08$} & \textcolor{black}{$\boldsymbol{80.69}${\scriptsize$\pm3.26$}} & \textcolor{black}{$\boldsymbol{78.91}${\scriptsize$\pm3.37$}} & \textcolor{black}{$\boldsymbol{76.96}${\scriptsize$\pm1.70$}} & $77.41${\scriptsize $\pm1.16$} & $42.09${\scriptsize $\pm0.91$}  \\
\textsc{MAX}         & $73.57${\scriptsize $\pm3.94$} & $78.83${\scriptsize $\pm1.70$} & $75.80${\scriptsize $\pm4.11$} & $75.96${\scriptsize $\pm1.82$} & $78.16${\scriptsize $\pm1.33$} & $177.48${\scriptsize $\pm4.70$}  \\
\textsc{MEAN}        & $73.13${\scriptsize $\pm3.18$} & $80.37${\scriptsize $\pm2.44$} & $76.89${\scriptsize $\pm2.23$} & $73.70${\scriptsize $\pm2.55$} & \textcolor{black}{$\boldsymbol{78.21}${\scriptsize$\pm0.90$}} & $177.49${\scriptsize $\pm4.69$}  \\
\midrule
\textsc{UFGPool-Sum} & \textcolor{red}{$\boldsymbol{77.77}${\scriptsize$\pm2.60$}} & \textcolor{violet}{$\boldsymbol{81.59}${\scriptsize$\pm1.40$}} & \textcolor{red}{$\boldsymbol{80.92}${\scriptsize$\pm1.68$}} & \textcolor{red}{$\boldsymbol{77.88}${\scriptsize$\pm1.24$}} & \textcolor{red}{$\boldsymbol{78.80}${\scriptsize $\pm0.56$}} & \textcolor{violet}{$\boldsymbol{41.74}${\scriptsize$\pm0.84$}}  \\ 
\textsc{UFGPool-Spectrum} & \textcolor{violet}{$\boldsymbol{77.23}${\scriptsize$\pm2.40$}} & \textcolor{red}{$\boldsymbol{82.05}${\scriptsize$\pm1.28$}} & \textcolor{violet}{$\boldsymbol{79.83}${\scriptsize$\pm1.88$}} & \textcolor{violet}{$\boldsymbol{77.54}${\scriptsize$\pm2.24$}} & \textcolor{violet}{$\boldsymbol{78.36}${\scriptsize $\pm0.77$}} & \textcolor{red}{$\boldsymbol{41.67}${\scriptsize$\pm0.95$}} \\ 
\bottomrule\\[-2.5mm]
\multicolumn{7}{l}{$\dagger$ The top three are highlighted by \textbf{\textcolor{red}{First}}, \textbf{\textcolor{violet}{Second}}, \textbf{Third}.}
\end{tabular}\vspace{-4mm}
\end{table*}

\paragraph{Dataset} 
We select six benchmarks to test the proposed pooling strategies, including four graph classification tasks with moderate sample sizes, one regression task, and one large-scale classification task. First five tasks use \textbf{TUDataset benchmarks} \cite{morris2020tudataset}, including  \textbf{D\&D} \cite{dobson2003distinguishing,shervashidze2011weisfeiler}, \textbf{PROTEINS} \cite{dobson2003distinguishing,borgwardt2005protein} to categorize proteins into enzyme and non-enzyme structures; \textbf{NCI1} \cite{wale2008comparison} to identify chemical compounds that block lung cancer cells; \textbf{Mutagenicity} \cite{kazius2005derivation,riesen2008iam} to recognize mutagenic molecular compounds for potentially marketable drug; and \textbf{QM7} \cite{blum,rupp} to predict atomization energy value of molecules. The dataset \textbf{ogbg-molhiv} \cite{hu2020open} is for large-scale molecule classification. 

\paragraph{Setup}
The network architecture for all baseline models is fixed to two convolutional layers followed by one pooling layer. The graph convolution for the five \textbf{TUDataset}s uses the \textsc{GCN} model, and for \textbf{ogbg-molhiv} uses \textsc{GIN} with virtual nodes \cite{ishiguro2019graph}. Given graph representations, the prediction is made by a two-layer MLP, where the hidden unit is identical to that of the convolutional layer. The hyperparameters (learning rate, weight decay, number of hidden units in each convolutional layer, and dropout ratio in the readout layer) are fine-tuned with grid search. 

Each dataset is split into training, validation and test sets by $80\%$, $10\%$ and $10\%$. The training stops when the validation loss stops improving for 20 consecutive epochs or reaching maximum $200$ epochs. All results are averaged over $10$ repetitions. The classification tasks report mean test accuracy for \textbf{TUDataset} and ROC-AUC score for \textbf{ogbg-molhiv}. The regression task on \textbf{QM7} reports mean square error (MSE).

\paragraph{Baseline}
We compare our framelet pooling (\textsc{UFGPool-SUM} and \textsc{UFGPool-Spectrum}) with six baseline methods that are capable for global pooling to verify the effectiveness of the learned graph representation. The baselines include \textsc{TopKPool} \cite{gao2019graph,cangea2018towards}, \textsc{AttentionPool} \cite{li2015gated}, \textsc{SAGPool} \cite{lee2019self}, as well as the classic \textsc{Sum}, \textsc{Mean} and \textsc{Max} pooling.

\paragraph{Results}
Table~\ref{table:poolingAblation} summarizes the performance comparison. Our \textsc{UFGPool} methods outperform other methods on all datasets. Specifically, \textsc{UFGPool-Sum} achieves the top accuracy in four out of six datasets, and the second best accuracy in the other two, where the top performance is achieved by \textsc{UFGPool-Spectrum}. We also observe that \textsc{UFGPool-Spectrum} performs better on small molecules prediction: \textbf{Mutagenicity}, \textbf{QM7} and \textbf{ogbg-molhiv}. This precedence might come from encoding the multi-scale signal energy to the network where the framelet spectra capture the practically significant features of molecular data.

\subsection{Sensitivity Analysis}
\label{sec:sensitity}
This section analyses the sensitivity of \textsc{UFGConv-R} and \textsc{UFGConv-S} on the hyperparameters dilation and scale level in the framelet system. The experiment is conducted on \textbf{Cora}, \textbf{Citeseer} and \textbf{Pubmed}. The dilation analysis selects values from $1.25$ to $4$ with step $0.25$, and the scale levels in those three datasets are fixed to $2$, $2$ and $3$ respectively. For the scale level analysis, the values are set from $1$ to $8$ with step $1$, and the dilation stays at default $2$. All other hyperparameters are tuned in the same way as Section~\ref{sec:experiment ufgconv}. We use shrinkage threshold $\sigma=1$ for \textsc{UFGConv-S}.

From Figure~\ref{fig:sensitive}, we can observe that changing dilation or scale level does not drastically impact on the accuracy for either method. In particular, the mean test accuracy is stable over all dilation values and reaches the peak with a small scale level ($2$ for \textbf{Cora} \& \textbf{Citeseer}; $3$ for \textbf{Pubmed}). For scale level $1$, the decreased accuracy is due to the insufficiency of scale and then not salient multiresolution. Thus, we can use dilation $2$ and scale level $2$ in practice, in which the GNN uses multi-scale framelet analysis to achieve supreme performance with a low computational cost.

\begin{figure}
    \centering
    \includegraphics[width=\columnwidth]{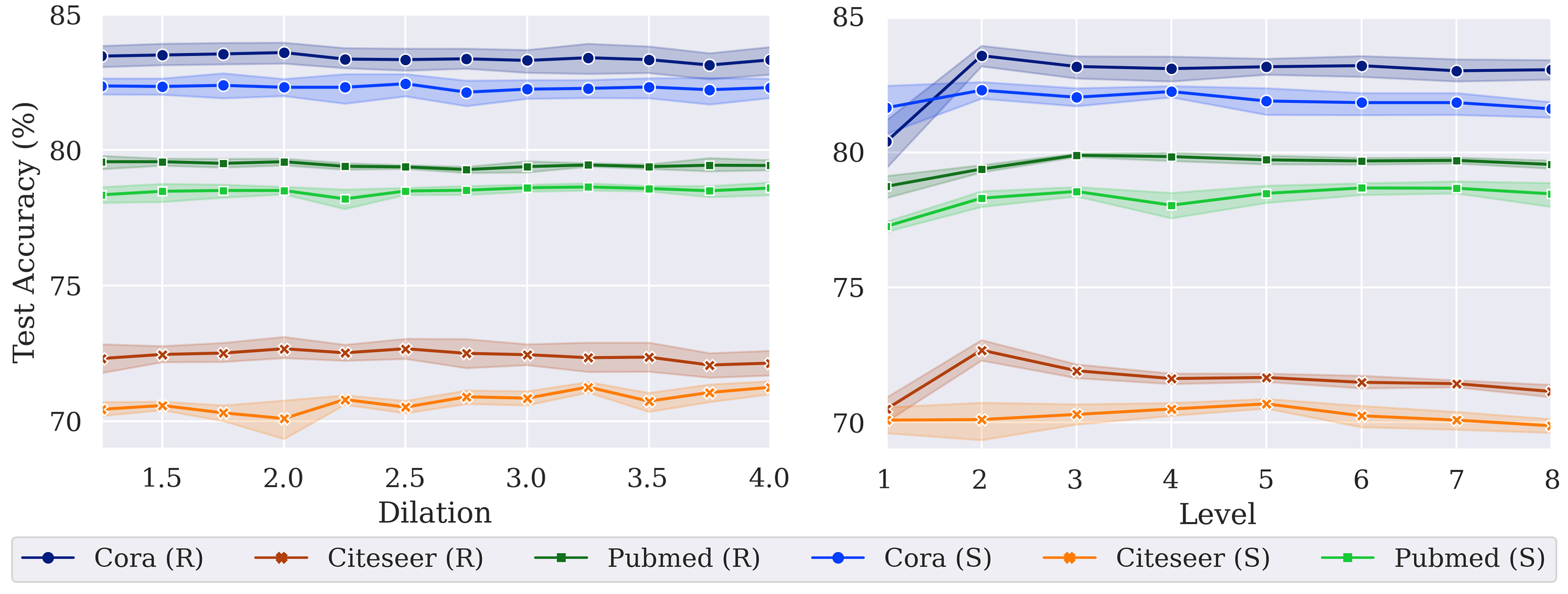}
    \vspace{-7mm}
    \caption{Sensitivity analysis for dilation (left) and scale level (right) with \textsc{UFGConv-R} and \textsc{UFGConv-S} on citation networks.}
    \label{fig:sensitive}\vspace{-6mm}
\end{figure}

\section{Conclusion}
We explore the adaptation of graph framelets for graph neural networks in this paper. As a multi-scale graph representation method, framelet transforms link graph neural networks and signal processing. In many node-level or graph-level tasks, framelet convolutions can reduce both feature and structure noises. We also introduce shrinkage activation that thresholds high-pass coefficients in framelet convolution, which strengthens the network denoising capability and simultaneously compresses graph signal at a remarkable rate. Moreover, we design graph pooling using framelet spectra at low and high passes. The proposed framelet convolutions with both ReLU and shrinkage surpass typical spatial-based and spectral-based graph convolutions on most benchmarks, and the framelet pooling outperforms baselines on a variety of graph property prediction tasks.

\section*{Acknowledgements}
YW and GM have been supported by the European Research Council (ERC) under the European Union’s Horizon 2020 research and innovation programme (grant n\textsuperscript{o} 757983). ML acknowledges supports from the National Natural Science Foundation of China (No. 61802132) and the ``Qianjiang Talent Program D'' of Zhejiang Province. This project was undertaken with the assistance of resources and services from the National Computational Infrastructure (NCI), which is supported by the Australian Government. This research was also undertaken with the assistance of resources and services from the GPU cluster Cobra of Max Planck Society. The authors would like to acknowledge support from the UNSW Resource Allocation Scheme managed by Research Technology Services at UNSW Sydney. The authors also acknowledge the technical assistance provided by the Sydney Informatics Hub, a Core Research Facility of the University of Sydney.

\FloatBarrier
\bibliographystyle{icml2021}
\bibliography{gnn}

\begin{thebibliography}{74}
\providecommand{\natexlab}[1]{#1}
\providecommand{\url}[1]{\texttt{#1}}
\expandafter\ifx\csname urlstyle\endcsname\relax
  \providecommand{\doi}[1]{doi: #1}\else
  \providecommand{\doi}{doi: \begingroup \urlstyle{rm}\Url}\fi

\bibitem[Balcilar et~al.(2021)Balcilar, Renton, H{\'e}roux, Ga{\"u}z{\`e}re,
  Adam, and Honeine]{balcilar2021analyzing}
Balcilar, M., Renton, G., H{\'e}roux, P., Ga{\"u}z{\`e}re, B., Adam, S., and
  Honeine, P.
\newblock Analyzing the expressive power of graph neural networks in a spectral
  perspective.
\newblock In \emph{ICLR}, 2021.

\bibitem[Behjat et~al.(2016)Behjat, Richter, Van De~Ville, and
  S{\"o}rnmo]{behjat2016signal}
Behjat, H., Richter, U., Van De~Ville, D., and S{\"o}rnmo, L.
\newblock Signal-adapted tight frames on graphs.
\newblock \emph{IEEE Transactions on Signal Processing}, 64\penalty0
  (22):\penalty0 6017--6029, 2016.

\bibitem[Blum \& Reymond(2009)Blum and Reymond]{blum}
Blum, L.~C. and Reymond, J.-L.
\newblock 970 million druglike small molecules for virtual screening in the
  chemical universe database {GDB-13}.
\newblock \emph{Journal of the American Chemical Society}, 131:\penalty0 8732,
  2009.

\bibitem[Borgwardt et~al.(2005)Borgwardt, Ong, Sch{\"o}nauer, Vishwanathan,
  Smola, and Kriegel]{borgwardt2005protein}
Borgwardt, K.~M., Ong, C.~S., Sch{\"o}nauer, S., Vishwanathan, S. V.~N., Smola,
  A.~J., and Kriegel, H.-P.
\newblock Protein function prediction via graph kernels.
\newblock \emph{Bioinformatics}, 21\penalty0 (suppl\_1):\penalty0 i47--i56,
  2005.

\bibitem[Brauchart et~al.(2015)Brauchart, Dick, Saff, Sloan, {Wang}, and
  Womersley]{Brauchart_etal2015}
Brauchart, J., Dick, J., Saff, E., Sloan, I., {Wang}, Y., and Womersley, R.
\newblock Covering of spheres by spherical caps and worst-case error for equal
  weight cubature in sobolev spaces.
\newblock \emph{Journal of Mathematical Analysis and Applications},
  431\penalty0 (2):\penalty0 782 -- 811, 2015.
\newblock ISSN 0022-247X.

\bibitem[Bruna et~al.(2014)Bruna, Zaremba, Szlam, and LeCun]{bruna2013spectral}
Bruna, J., Zaremba, W., Szlam, A., and LeCun, Y.
\newblock Spectral networks and locally connected networks on graphs.
\newblock In \emph{ICLR}, 2014.

\bibitem[Cai et~al.(2012)Cai, Dong, Osher, and Shen]{cai2012image}
Cai, J.-F., Dong, B., Osher, S., and Shen, Z.
\newblock Image restoration: total variation, wavelet frames, and beyond.
\newblock \emph{Journal of the American Mathematical Society}, 25\penalty0
  (4):\penalty0 1033--1089, 2012.

\bibitem[Cangea et~al.(2018)Cangea, Veli{\v{c}}kovi{\'c}, Jovanovi{\'c}, Kipf,
  and Li{\`o}]{cangea2018towards}
Cangea, C., Veli{\v{c}}kovi{\'c}, P., Jovanovi{\'c}, N., Kipf, T., and Li{\`o},
  P.
\newblock Towards sparse hierarchical graph classifiers.
\newblock In \emph{NeurIPS Workshop on Relational Representation Learning},
  2018.

\bibitem[Chen et~al.(2020)Chen, Wei, Huang, Ding, and Li]{chen2020simple}
Chen, M., Wei, Z., Huang, Z., Ding, B., and Li, Y.
\newblock Simple and deep graph convolutional networks.
\newblock In \emph{ICML}, pp.\  1725--1735. PMLR, 2020.

\bibitem[Cheung et~al.(2018)Cheung, Magli, Tanaka, and Ng]{cheung2018graph}
Cheung, G., Magli, E., Tanaka, Y., and Ng, M.~K.
\newblock Graph spectral image processing.
\newblock \emph{Proceedings of the IEEE}, 106\penalty0 (5):\penalty0 907--930,
  2018.

\bibitem[Chui et~al.(2015)Chui, Filbir, and Mhaskar]{ChFiMh2015}
Chui, C.~K., Filbir, F., and Mhaskar, H.~N.
\newblock Representation of functions on big data: graphs and trees.
\newblock \emph{Applied and Computational Harmonic Analysis}, 38\penalty0
  (3):\penalty0 489--509, 2015.

\bibitem[Cohen et~al.(1993)Cohen, Daubechies, Jawerth, and
  Vial]{cohen1993multiresolution}
Cohen, A., Daubechies, I., Jawerth, B., and Vial, P.
\newblock Multiresolution analysis, wavelets and fast algorithms on an
  interval.
\newblock \emph{Comptes rendus de l'Acad{\'e}mie des sciences. S{\'e}rie 1,
  Math{\'e}matique}, 316\penalty0 (5):\penalty0 417--421, 1993.

\bibitem[Crovella \& Kolaczyk(2003)Crovella and Kolaczyk]{CrKo2003}
Crovella, M. and Kolaczyk, E.
\newblock Graph wavelets for spatial traffic analysis.
\newblock In \emph{IEEE INFOCOM 2003}, volume~3, pp.\  1848--1857. IEEE, 2003.

\bibitem[Defferrard et~al.(2016{\natexlab{a}})Defferrard, Bresson, and
  Vandergheynst]{DeBrVa2016}
Defferrard, M., Bresson, X., and Vandergheynst, P.
\newblock Convolutional neural networks on graphs with fast localized spectral
  filtering.
\newblock In \emph{NIPS}, pp.\  3844--3852, 2016{\natexlab{a}}.

\bibitem[Defferrard et~al.(2016{\natexlab{b}})Defferrard, Bresson, and
  Vandergheynst]{defferrard2016convolutional}
Defferrard, M., Bresson, X., and Vandergheynst, P.
\newblock Convolutional neural networks on graphs with fast localized spectral
  filtering.
\newblock In \emph{NIPS}, pp.\  3844--3852, 2016{\natexlab{b}}.

\bibitem[Deng et~al.(2020)Deng, Zhao, Wang, Zhang, and Feng]{Deng2020GraphZoom}
Deng, C., Zhao, Z., Wang, Y., Zhang, Z., and Feng, Z.
\newblock Graphzoom: A multi-level spectral approach for accurate and scalable
  graph embedding.
\newblock In \emph{ICLR}, 2020.

\bibitem[Dobson \& Doig(2003)Dobson and Doig]{dobson2003distinguishing}
Dobson, P.~D. and Doig, A.~J.
\newblock Distinguishing enzyme structures from non-enzymes without alignments.
\newblock \emph{Journal of Molecular Biology}, 330\penalty0 (4):\penalty0
  771--783, 2003.

\bibitem[Dong(2017)]{dong2017sparse}
Dong, B.
\newblock Sparse representation on graphs by tight wavelet frames and
  applications.
\newblock \emph{Applied and Computational Harmonic Analysis}, 42\penalty0
  (3):\penalty0 452--479, 2017.

\bibitem[Dong \& Shen(2013)Dong and Shen]{Dong2013mra}
Dong, B. and Shen, Z.
\newblock \emph{MRA-Based Wavelet Frames and Applications}.
\newblock IAS Lecture Note Series, 06 2013.

\bibitem[Donoho(1992)]{donoho1992wavelet}
Donoho, D.~L.
\newblock Wavelet shrinkage and wvd: a 10-minute tour.
\newblock In \emph{Presented on the International Conference on Wavelets and
  Applications, Tolouse, France}, 1992.

\bibitem[Donoho(1995)]{donoho1995noising}
Donoho, D.~L.
\newblock De-noising by soft-thresholding.
\newblock \emph{IEEE Transactions on Information Theory}, 41\penalty0
  (3):\penalty0 613--627, 1995.

\bibitem[Donoho \& Johnstone(1994)Donoho and Johnstone]{donoho1994ideal}
Donoho, D.~L. and Johnstone, J.~M.
\newblock Ideal spatial adaptation by wavelet shrinkage.
\newblock \emph{Biometrika}, 81\penalty0 (3):\penalty0 425--455, 1994.

\bibitem[Donoho et~al.(1995)Donoho, Johnstone, Kerkyacharian, and
  Picard]{donoho1995wavelet}
Donoho, D.~L., Johnstone, I.~M., Kerkyacharian, G., and Picard, D.
\newblock Wavelet shrinkage: {A}symptopia?
\newblock \emph{Journal of the Royal Statistical Society: Series B
  (Methodological)}, 57\penalty0 (2):\penalty0 301--337, 1995.

\bibitem[Efron \& Morris(1975)Efron and Morris]{efron1975data}
Efron, B. and Morris, C.
\newblock Data analysis using {Stein's} estimator and its generalizations.
\newblock \emph{Journal of the American Statistical Association}, 70\penalty0
  (350):\penalty0 311--319, 1975.

\bibitem[Figueiredo \& Nowak(2003)Figueiredo and
  Nowak]{figueiredo2003algorithm}
Figueiredo, M.~A. and Nowak, R.~D.
\newblock An {EM} algorithm for wavelet-based image restoration.
\newblock \emph{IEEE Transactions on Image Processing}, 12\penalty0
  (8):\penalty0 906--916, 2003.

\bibitem[Gao \& Ji(2019)Gao and Ji]{gao2019graph}
Gao, H. and Ji, S.
\newblock Graph {U}-nets.
\newblock In \emph{ICML}, 2019.

\bibitem[Gavish et~al.(2010)Gavish, Nadler, and Coifman]{GaNaCo2010}
Gavish, M., Nadler, B., and Coifman, R.~R.
\newblock {Multiscale wavelets on trees, graphs and high dimensional data:
  theory and applications to semi supervised learning}.
\newblock In \emph{ICML}, pp.\  367--374, 2010.

\bibitem[Gilmer et~al.(2017)Gilmer, Schoenholz, Riley, Vinyals, and
  Dahl]{Gilmer_etal2017}
Gilmer, J., Schoenholz, S.~S., Riley, P.~F., Vinyals, O., and Dahl, G.~E.
\newblock Neural message passing for quantum chemistry.
\newblock In \emph{ICML}, pp.\  1263--1272, 2017.

\bibitem[Glorot \& Bengio(2010)Glorot and Bengio]{glorot2010understanding}
Glorot, X. and Bengio, Y.
\newblock Understanding the difficulty of training deep feedforward neural
  networks.
\newblock In \emph{AISTATS}, pp.\  249--256, 2010.

\bibitem[Grover \& Leskovec(2016)Grover and Leskovec]{grover2016node2vec}
Grover, A. and Leskovec, J.
\newblock node2vec: Scalable feature learning for networks.
\newblock In \emph{ACM SIGKDD}, pp.\  855--864, 2016.

\bibitem[Hamilton et~al.(2017)Hamilton, Ying, and
  Leskovec]{hamilton2017inductive}
Hamilton, W., Ying, Z., and Leskovec, J.
\newblock Inductive representation learning on large graphs.
\newblock In \emph{NIPS}, pp.\  1024--1034, 2017.

\bibitem[Hammond et~al.(2011)Hammond, Vandergheynst, and Gribonval]{HaVaRe2011}
Hammond, D.~K., Vandergheynst, P., and Gribonval, R.
\newblock Wavelets on graphs via spectral graph theory.
\newblock \emph{Applied and Computational Harmonic Analysis}, 30\penalty0
  (2):\penalty0 129--150, 2011.

\bibitem[Hu et~al.(2020)Hu, Fey, Zitnik, Dong, Ren, Liu, Catasta, and
  Leskovec]{hu2020open}
Hu, W., Fey, M., Zitnik, M., Dong, Y., Ren, H., Liu, B., Catasta, M., and
  Leskovec, J.
\newblock Open graph benchmark: Datasets for machine learning on graphs.
\newblock In \emph{NeurIPS}, 2020.

\bibitem[Huang et~al.(2021)Huang, He, Singh, Lim, and
  Benson]{huang2020combining}
Huang, Q., He, H., Singh, A., Lim, S.-N., and Benson, A.~R.
\newblock Combining label propagation and simple models out-performs graph
  neural networks.
\newblock \emph{ICLR}, 2021.

\bibitem[Ioffe \& Szegedy(2015)Ioffe and Szegedy]{ioffe2015batch}
Ioffe, S. and Szegedy, C.
\newblock Batch normalization: Accelerating deep network training by reducing
  internal covariate shift.
\newblock In \emph{ICML}, pp.\  448--456, 2015.

\bibitem[Ishiguro et~al.(2019)Ishiguro, Maeda, and Koyama]{ishiguro2019graph}
Ishiguro, K., Maeda, S.-i., and Koyama, M.
\newblock Graph warp module: an auxiliary module for boosting the power of
  graph neural networks.
\newblock \emph{arXiv:1902.01020}, 2019.

\bibitem[{ISO}(2019)]{jp2}
{ISO}.
\newblock {ISO/IEC 15444-1:2019 Information technology - JPEG 2000 image coding
  system - Part 1: Core coding system}.
\newblock 2019.

\bibitem[Kazius et~al.(2005)Kazius, McGuire, and Bursi]{kazius2005derivation}
Kazius, J., McGuire, R., and Bursi, R.
\newblock Derivation and validation of toxicophores for mutagenicity
  prediction.
\newblock \emph{Journal of Medicinal Chemistry}, 48\penalty0 (1):\penalty0
  312--320, 2005.

\bibitem[Kipf \& Welling(2017)Kipf and Welling]{KiWe2017}
Kipf, T.~N. and Welling, M.
\newblock Semi-supervised classification with graph convolutional networks.
\newblock In \emph{ICLR}, 2017.

\bibitem[Knyazev et~al.(2019)Knyazev, Taylor, and
  Amer]{knyazev2019understanding}
Knyazev, B., Taylor, G.~W., and Amer, M.~R.
\newblock Understanding attention and generalization in graph neural networks.
\newblock In \emph{NeurIPS}, 2019.

\bibitem[Lee et~al.(2019)Lee, Lee, and Kang]{lee2019self}
Lee, J., Lee, I., and Kang, J.
\newblock Self-attention graph pooling.
\newblock In \emph{ICML}, 2019.

\bibitem[Li et~al.(2020{\natexlab{a}})Li, Xiong, Thabet, and
  Ghanem]{li2020deepergcn}
Li, G., Xiong, C., Thabet, A., and Ghanem, B.
\newblock {DeeperGCN}: All you need to train deeper {GCN}s.
\newblock \emph{arXiv:2006.07739}, 2020{\natexlab{a}}.

\bibitem[Li et~al.(2020{\natexlab{b}})Li, Ma, Wang, and Zhuang]{li2020fast}
Li, M., Ma, Z., Wang, Y.~G., and Zhuang, X.
\newblock Fast haar transforms for graph neural networks.
\newblock \emph{Neural Networks}, pp.\  188--198, 2020{\natexlab{b}}.

\bibitem[Li et~al.(2016)Li, Tarlow, Brockschmidt, and Zemel]{li2015gated}
Li, Y., Tarlow, D., Brockschmidt, M., and Zemel, R.
\newblock Gated graph sequence neural networks.
\newblock In \emph{ICLR}, 2016.

\bibitem[Liao et~al.(2019)Liao, Zhao, Urtasun, and Zemel]{liao2018lanczosnet}
Liao, R., Zhao, Z., Urtasun, R., and Zemel, R.
\newblock Lanczosnet: Multi-scale deep graph convolutional networks.
\newblock In \emph{ICLR}, 2019.

\bibitem[Ma et~al.(2020)Ma, Xuan, Wang, Li, and Li\`{o}]{Ma2020}
Ma, Z., Xuan, J., Wang, Y.~G., Li, M., and Li\`{o}, P.
\newblock Path integral based convolution and pooling for graph neural
  networks.
\newblock In \emph{NeurIPS}, volume~33, pp.\  16433--16445, 2020.

\bibitem[Maggioni \& Mhaskar(2008)Maggioni and Mhaskar]{MaMh2008}
Maggioni, M. and Mhaskar, H.~H.
\newblock Diffusion polynomial frames on metric measure spaces.
\newblock \emph{Applied and Computational Harmonic Analysis}, 24\penalty0
  (3):\penalty0 329 -- 353, 2008.

\bibitem[Mallat(1989)]{mallat1989theory}
Mallat, S.~G.
\newblock A theory for multiresolution signal decomposition: the wavelet
  representation.
\newblock \emph{IEEE Transactions on Pattern Analysis and Machine
  Intelligence}, 11\penalty0 (7):\penalty0 674--693, 1989.

\bibitem[{Milanfar}(2013)]{milanfar2013tour}
{Milanfar}, P.
\newblock A tour of modern image filtering: New insights and methods, both
  practical and theoretical.
\newblock \emph{IEEE Signal Processing Magazine}, 30\penalty0 (1):\penalty0
  106--128, 2013.

\bibitem[Morris et~al.(2020)Morris, Kriege, Bause, Kersting, Mutzel, and
  Neumann]{morris2020tudataset}
Morris, C., Kriege, N.~M., Bause, F., Kersting, K., Mutzel, P., and Neumann, M.
\newblock {TUDataset: A collection of benchmark datasets for learning with
  graphs}.
\newblock In \emph{ICML Workshop ``Graph Representation Learning and Beyond''},
  2020.

\bibitem[Perozzi et~al.(2014)Perozzi, Al-Rfou, and Skiena]{perozzi2014deepwalk}
Perozzi, B., Al-Rfou, R., and Skiena, S.
\newblock Deep{W}alk: Online learning of social representations.
\newblock In \emph{KDD}, pp.\  701--710, 2014.

\bibitem[Rahaman et~al.(2019)Rahaman, Baratin, Arpit, Draxler, Lin, Hamprecht,
  Bengio, and Courville]{rahaman2019spectral}
Rahaman, N., Baratin, A., Arpit, D., Draxler, F., Lin, M., Hamprecht, F.,
  Bengio, Y., and Courville, A.
\newblock On the spectral bias of neural networks.
\newblock In \emph{ICML}, pp.\  5301--5310. PMLR, 2019.

\bibitem[Riesen \& Bunke(2008)Riesen and Bunke]{riesen2008iam}
Riesen, K. and Bunke, H.
\newblock Iam graph database repository for graph based pattern recognition and
  machine learning.
\newblock In \emph{Joint IAPR International Workshops on Statistical Techniques
  in Pattern Recognition (SPR) and Structural and Syntactic Pattern Recognition
  (SSPR)}, pp.\  287--297. Springer, 2008.

\bibitem[Rossi et~al.(2020)Rossi, Frasca, Chamberlain, Eynard, Bronstein, and
  Monti]{rossi2020sign}
Rossi, E., Frasca, F., Chamberlain, B., Eynard, D., Bronstein, M., and Monti,
  F.
\newblock {SIGN}: Scalable inception graph neural networks.
\newblock In \emph{ICML Workshop ``Graph Representation Learning and Beyond''},
  2020.

\bibitem[Rupp et~al.(2012)Rupp, Tkatchenko, M\"uller, and von Lilienfeld]{rupp}
Rupp, M., Tkatchenko, A., M\"uller, K.-R., and von Lilienfeld, O.~A.
\newblock Fast and accurate modeling of molecular atomization energies with
  machine learning.
\newblock \emph{Physical Review Letters}, 108:\penalty0 058301, 2012.

\bibitem[Shen(2010)]{Shen2010wavelet}
Shen, Z.
\newblock Wavelet frames and image restorations.
\newblock In \emph{the International Congress of Mathematicians (ICM)}, pp.\
  2834--2863, 2010.

\bibitem[Shervashidze et~al.(2011)Shervashidze, Schweitzer, Van~Leeuwen,
  Mehlhorn, and Borgwardt]{shervashidze2011weisfeiler}
Shervashidze, N., Schweitzer, P., Van~Leeuwen, E.~J., Mehlhorn, K., and
  Borgwardt, K.~M.
\newblock Weisfeiler-{L}ehman graph kernels.
\newblock \emph{Journal of Machine Learning Research}, 12\penalty0
  (77):\penalty0 2539--2561, 2011.

\bibitem[Singer et~al.(2009)Singer, Shkolnisky, and
  Nadler]{singer2009diffusion}
Singer, A., Shkolnisky, Y., and Nadler, B.
\newblock Diffusion interpretation of nonlocal neighborhood filters for signal
  denoising.
\newblock \emph{SIAM Journal on Imaging Sciences}, 2\penalty0 (1):\penalty0
  118--139, 2009.

\bibitem[Tibshirani(1996)]{tibshirani1996regression}
Tibshirani, R.
\newblock Regression shrinkage and selection via the lasso.
\newblock \emph{Journal of the Royal Statistical Society: Series B
  (Methodological)}, 58\penalty0 (1):\penalty0 267--288, 1996.

\bibitem[Veli{\v{c}}kovi{\'c} et~al.(2018)Veli{\v{c}}kovi{\'c}, Cucurull,
  Casanova, Romero, Lio, and Bengio]{velivckovic2017graph}
Veli{\v{c}}kovi{\'c}, P., Cucurull, G., Casanova, A., Romero, A., Lio, P., and
  Bengio, Y.
\newblock Graph attention networks.
\newblock In \emph{ICLR}, 2018.

\bibitem[Vignac et~al.(2020)Vignac, Loukas, and Frossard]{Vignac2020}
Vignac, C., Loukas, A., and Frossard, P.
\newblock Building powerful and equivariant graph neural networks with
  structural message-passing.
\newblock In \emph{NeurIPS}, volume~33, pp.\  14154--14166, 2020.

\bibitem[Wale et~al.(2008)Wale, Watson, and Karypis]{wale2008comparison}
Wale, N., Watson, I.~A., and Karypis, G.
\newblock Comparison of descriptor spaces for chemical compound retrieval and
  classification.
\newblock \emph{Knowledge and Information Systems}, 14\penalty0 (3):\penalty0
  347--375, 2008.

\bibitem[Wang et~al.(2020{\natexlab{a}})Wang, Zhu, Bo, Cui, Shi, and
  Pei]{wang2020gcn}
Wang, X., Zhu, M., Bo, D., Cui, P., Shi, C., and Pei, J.
\newblock Am-gcn: Adaptive multi-channel graph convolutional networks.
\newblock In \emph{ACM SIGKDD}, pp.\  1243--1253, 2020{\natexlab{a}}.

\bibitem[Wang \& Zhuang(2018)Wang and Zhuang]{WaZh2018}
Wang, Y.~G. and Zhuang, X.
\newblock Tight framelets and fast framelet filter bank transforms on
  manifolds.
\newblock \emph{Applied and Computational Harmonic Analysis}, 2018.

\bibitem[Wang \& Zhuang(2019)Wang and Zhuang]{wang2019tight}
Wang, Y.~G. and Zhuang, X.
\newblock {Tight framelets on graphs for multiscale data analysis}.
\newblock In \emph{Wavelets and Sparsity XVIII}, volume 11138, pp.\  100 --
  111. SPIE, 2019.

\bibitem[Wang et~al.(2020{\natexlab{b}})Wang, Li, Ma, Montufar, Zhuang, and
  Fan]{wang2020haar}
Wang, Y.~G., Li, M., Ma, Z., Montufar, G., Zhuang, X., and Fan, Y.
\newblock Haar graph pooling.
\newblock In \emph{ICML}, pp.\  9952--9962. PMLR, 2020{\natexlab{b}}.

\bibitem[{Wu} et~al.(2021){Wu}, {Pan}, {Chen}, {Long}, {Zhang}, and
  {Yu}]{wu2021comprehensive}
{Wu}, Z., {Pan}, S., {Chen}, F., {Long}, G., {Zhang}, C., and {Yu}, P.~S.
\newblock A comprehensive survey on graph neural networks.
\newblock \emph{IEEE Transactions on Neural Networks and Learning Systems},
  32\penalty0 (1):\penalty0 4--24, 2021.

\bibitem[Xu et~al.(2019{\natexlab{a}})Xu, Shen, Cao, Qiu, and
  Cheng]{xu2019graph}
Xu, B., Shen, H., Cao, Q., Qiu, Y., and Cheng, X.
\newblock Graph wavelet neural network.
\newblock In \emph{ICLR}, 2019{\natexlab{a}}.

\bibitem[Xu et~al.(2020)Xu, Ma, Liu, Deb, Liu, Tang, and
  Jain]{xu2020adversarial}
Xu, H., Ma, Y., Liu, H.-C., Deb, D., Liu, H., Tang, J.-L., and Jain, A.~K.
\newblock Adversarial attacks and defenses in images, graphs and text: A
  review.
\newblock \emph{International Journal of Automation and Computing}, 17\penalty0
  (2):\penalty0 151--178, 2020.

\bibitem[Xu et~al.(2019{\natexlab{b}})Xu, Hu, Leskovec, and Jegelka]{xu2018how}
Xu, K., Hu, W., Leskovec, J., and Jegelka, S.
\newblock How powerful are graph neural networks?
\newblock In \emph{ICLR}, 2019{\natexlab{b}}.

\bibitem[Zhang et~al.(2018{\natexlab{a}})Zhang, Shi, Xie, Ma, King, and
  Yeung]{zhang2018gaan}
Zhang, J., Shi, X., Xie, J., Ma, H., King, I., and Yeung, D.-Y.
\newblock {GaAN: Gated attention networks for learning on large and
  spatiotemporal graphs}.
\newblock In \emph{UAI}, pp.\  339--349, 2018{\natexlab{a}}.

\bibitem[Zhang et~al.(2018{\natexlab{b}})Zhang, Cui, Neumann, and
  Chen]{zhang2018end}
Zhang, M., Cui, Z., Neumann, M., and Chen, Y.
\newblock An end-to-end deep learning architecture for graph classification.
\newblock In \emph{AAAI}, 2018{\natexlab{b}}.

\bibitem[Zheng et~al.(2020)Zheng, Zhou, Li, Wang, and Gao]{zheng2020graph}
Zheng, X., Zhou, B., Li, M., Wang, Y.~G., and Gao, J.
\newblock {MathNet: Haar-like} wavelet multiresolution-analysis for graph
  representation and learning.
\newblock \emph{arXiv:2007.11202}, 2020.

\bibitem[Zheng et~al.(2021)Zheng, Zhou, Wang, and Zhuang]{zheng2020decimated}
Zheng, X., Zhou, B., Wang, Y.~G., and Zhuang, X.
\newblock Decimated framelet system on graphs and fast {G}-framelet transforms.
\newblock \emph{Journal of Machine Learning Research}, 2021.

\end{thebibliography}

\newpage
\onecolumn
\appendix
\section{Undecimated Framelets on Graph}\label{append:ufg}
An undirected and weighted graph $\gph$ is an ordered triple $\gph=(V,\EG,\wG)$ with a finite vertex set $V$, and edge set $\EG\subseteq V\times V$, and a non-negative edge weight function $\wG: \EG\rightarrow\R$. We consider the $l_2$ space on the graph $\gph$, $l_2(\gph)$, with inner product 
\begin{equation*}
    \ipG{f,g}:=\sum_{v\in V} f(v)\conj{g(v)}, \quad f,g\in l_2(\gph),
\end{equation*}
where $\conj{g}$ is the complex conjugate to $g$, and induced norm $\nrmG{f} := \sqrt{\ipG{f,f}}$. Let $N$ be the number of vertices for $\gph$.  

The construction of framelets uses the graph spectrum, and filter bank that is a set of \emph{filters}.
A filter (or mask) is complex-valued sequence on the graph $\mask:=\{\mask_k\}_{k\in\Z}\subseteq \C$ satisfying $|h_k|<\infty$.
The \emph{Fourier series} of a sequence $\{\mask_k\}_{k\in\Z}$ is the $1$-periodic function $\FT{\mask}(\xi):=\sum_{k\in\Z}\mask_k e^{-2\pi i k\xi}$, $\xi\in\R$.
The \emph{scaling functions} $\Psi_j=\{\scala; \scalb^{(1)},\ldots,\scalb^{(n)}\}$ associated with the filter bank $\filtbk:=\{\maska; \maskb[1],\ldots,\maskb[n]\}$ are complex-valued functions on the real axis, which satisfy the equations, for $r=1,\ldots,n$, $\xi\in\Rone$,
\begin{equation}
\label{eq:refinement}
    \FT{\scala}(2\xi) = \FS{\maska}(\xi)\FT{\scala}(\xi),\quad
    \FT{\scalb^{(r)}}(2\xi) = \FS{\maskb}(\xi)\FT{\scala}(\xi).
\end{equation}
Here the $\maska$ is called \emph{low-pass filter} and $\maskb$, $n=1,\ldots,r$ are \emph{high-pass filters}.
Let $\{(\eigfm,\eigvm)\}_{\ell=1}^\NV$ the eigen-pair for the graph Laplacian $\gL$ on $l_2(\gph)$.
For $j\in\Z$ and $\uG\in V$, the \emph{undecimated framelets} $\cfra(g)$ and $\cfrb{r}(g)$, $v\in V$ at scale $j$ are \emph{filtered Bessel kernels} (or summability kernels)
\begin{equation}\label{defn:ufra:ufrb}
\begin{aligned}
    \cfra(g) :=&  \sum_{\ell=1}^{\NV} \FT{\scala}\left(\frac{\eigvm}{2^{j}}\right)\conj{\eigfm(\uG)}\eigfm(v),\\
    \cfrb{r}(g) :=&  \sum_{\ell=1}^{\NV} \FT{\scalb^{(r)}}\left(\frac{\eigvm}{2^{j}}\right)\conj{\eigfm(\uG)}\eigfm(v), \quad r = 1,\ldots,n.
\end{aligned}
\end{equation}
See for example, \citet{Brauchart_etal2015,MaMh2008}.
Here, $j$ and $\uG$ in $\cfra(g)$ and $\cfrb{r}(g)$ are the ``dilation'' at scale $j$ and the ``translation'' at a vertex $\uG\in V$, which have their counterpart in the traditional wavelets on the real axis.
For two integers $J, J_1$ satisfying $J > J_1$, we define an \emph{undecimated framelet system} $\ufrsys[]\left(\Psi,\filtbk;\gph\right)$ (starting from a scale $J_1$) as a non-homogeneous, stationary affine system:   
\begin{equation}\label{defn:UFS}
\begin{aligned}
   \ufrsys[J_1]^{J}(\Psi,\filtbk)
   &:=\ufrsys[J_1]^J(\Psi,\filtbk;\gph) \\
   &:=\{\cfra[J_1,\uG] \setsep \uG\in V\} \cup
   \{\cfrb{r} \setsep \uG\in V, j= J_1,\ldots, J\}_{r = 1}^n.
\end{aligned}
\end{equation}
The system $\ufrsys[J_1]^{J}(\Psi,\filtbk)$ is then called an \emph{undecimated tight frame} for $l_2(\gph)$ and the elements in $\ufrsys[J_1]^{J}(\Psi,\filtbk)$ are called \emph{undecimated tight framelets} on $\gph$.

There are many types of filters. In the experiments, we use the Haar-type filter with one high pass: for $x\in\R$,
\begin{equation*}
    \FT{\maska}(x) = \cos(x/2)\mbox{~~and~~} \FT{\maskb[1]}(x) = \sin(x/2).
\end{equation*}

\section{Equivalence Conditions of Tightness of Framelet System}
The framelet transforms between time domain and framelet domain can attain zero loss. It is due to the \emph{tightness} of the undecimated framelet system. That is, the \eqref{eq:f:UFS0} holds for all $f\in l_2(\gph)$. The tightness of a framelet system is met by an appropriate choice of the filter bank. In our case, the classical filter bank that satisfies the partition of unity would guarantee the tightness of the framelet system, which are equivalent conditions (iv) and (v) in Theorem~\ref{thm:UFS} (see below). The theorem was proved in \citet{zheng2020decimated}. For completeness, we give a brief proof here.
\begin{theorem}[Equivalence of Framelet Tightness]\label{thm:UFS}
Let $\gph=(g,\EG,\wG)$ be a graph and $\{(\eigfm,\eigvm)\}_{\ell=1}^\NV$ a set of orthonormal eigen-pairs for $l_2(\gph)$. Let $\Psi=\{\scala;\scalb^{(1)},\ldots,\scalb^{(n)}\}$ be a set of functions in $L_1(\R)$ associated with a filter bank $\filtbk=\{\maska;\maskb[1],\ldots,\maskb[n]\}$ satisfying \eqref{eq:refinement}. Let integer $J\geq1$, and $\ufrsys[J_1]^{J}(\Psi,\filtbk;\gph), J_1=1,\ldots, J$, be an undecimated framelet system given in \eqref{defn:UFS} with framelets $\cfra$ and $\cfrb{n}$ in \eqref{defn:ufra:ufrb}. Then, the following statements are equivalent.
\begin{itemize}
\item[(i)] For each $J_1=1,\ldots,J$, the undecimated framelet system $\ufrsys[J_1]^{J}(\Psi,\filtbk;\gph)$ is a tight frame for $l_2(\gph)$, that is, $\forall f\in l_2(\gph)$,
    \begin{equation}\label{eq:f:UFS0}
    \nrmG{f}^2 =\sum_{\uG\in V}\Big|\ipG{f,\cfra[J_1,\uG]}\Big|^2
    +\sum_{j=J_1}^{J}\sum_{r=1}^n\sum_{\uG\in V}\Big|\ipG{f,\cfrb{r}}\Big|^2.
    \end{equation}

\item[(ii)]  For all $f\in l_2(\gph)$ and for $j=1,\ldots,J-1$, the following identities hold:
\begin{align}
&   f = \sum_{\uG\in V} \ipG{f,\cfra[J,\uG]}\cfra[J,\uG]
+\sum_{r=1}^n\sum_{\uG\in V}\ipG{f,\cfrb[J,\uG]{r}}\cfrb[J,\uG]{r},
\label{thmeq:normalization1}\\
& \sum_{\uG\in V} \ipG{f,\cfra[j+1,\uG]}\cfra[j+1,\uG]
= \sum_{\uG\in V} \ipG{f,\cfra}\cfra+
\sum_{r=1}^{n}\sum_{\uG\in V} \ipG{f,\cfrb{r}}\cfrb{r}. \label{thmeq:2scale1}
\end{align}

\item[(iii)] For  all $f\in l_2(\gph)$ and for $j=1,\ldots, J-1$, the following identities hold:
\begin{align}
           & \nrmG{f}^2 = \sum_{\uG\in V} \bigl|\ipG{f,\cfra[J,\uG]}\bigr|^{2}
           +\sum_{r=1}^n\sum_{\uG\in V}\bigl|\ipG{f,\cfrb[J,\uG]{r}}\bigr|^{2}, \quad\label{thmeq:normalization2}\\
            & \sum_{\uG\in V} \bigl|\ipG{f,\cfra[j+1,\uG]}\bigr|^{2}
            = \sum_{\uG\in V} \bigl|\ipG{f,\cfra}\bigr|^{2} + \sum_{r=1}^{n}\sum_{\uG\in V} \bigl|\ipG{f,\cfrb{r}}\bigr|^{2}.&\label{thmeq:2scale2}
\end{align}

\item[(iv)] The  functions in $\Psi$ satisfy
\begin{align}
   1 = \left|\FT{\scala}\left(\frac{\eigvm}{2^{J}}\right)\right|^{2} + \sum_{r=1}^{n}\left|\FT{\scalb^{(r)}}\left(\frac{\eigvm}{2^{J}}\right)\right|^{2} &\quad \forall
  \ell=1,\ldots,\NV, \label{thmeq:nrm:alpha:beta}\\
     \left|\FT{\scala}\left(\frac{\eigvm}{2^{j+1}}\right)\right|^{2}
    = \left|\FT{\scala}\left(\frac{\eigvm}{2^{j}}\right)\right|^{2} + \sum_{r=1}^{n}\left|\FT{\scalb^{(r)}}\left(\frac{\eigvm}{2^{j}}\right)\right|^{2} &\quad \forall
 \begin{array}{l}
 \ell=1,\ldots,\NV,\\
 j=1,\ldots,J-1.
 \end{array}\label{thmeq:2scale:alpha:beta}
 \end{align}

\item[(v)] The identities in \eqref{thmeq:nrm:alpha:beta} hold  and  the filters in the filter bank $\filtbk$ satisfy
\begin{align}
  \left|\FS{\maska}\left(\frac{\eigvm}{2^{j}}\right)\right|^{2} + \sum_{n=1}^{r} \left|\FS{\maskb}\left(\frac{\eigvm}{2^{j}}\right)\right|^{2} = 1 \quad \forall \ell\in\sigma_{\scala}^{(j)},\; j = 2,\ldots,J,
\label{thmeq:2scale:masks}
 \end{align}
with
\[
\sigma_{\scala}^{(j)}:=\left\{\ell\in\{1,\ldots,\NV\} \setsep \FT{\scala}\left(\frac{\eigvm}{2^j}\right) \neq 0\right\}.
\]
\end{itemize}
\end{theorem}
\begin{proof}
(i)$\Longleftrightarrow$(ii). Let $\cfrspa{j}:=\mathrm{span}\{\cfra: \uG\in V\}$ and $\cfrspb{j}:=\mathrm{span}\{\cfrb{r}: \uG\in V\}$. Define projections  $\cfrpra{j}, \cfrprb{j}$, $r=1,\dots,n$ by
\begin{equation}\label{eqs:proj.ufr}
    \cfrpra{j}(f) := \sum_{\uG\in V} \ipG{f,\cfra}\cfra,\quad
    \cfrprb{j}(f) := \sum_{\uG\in V} \ipG{f,\cfrb{r}}\cfrb{r},\quad f\in l_2(\gph).
\end{equation}
Since $\ufrsys[J_1]^{J}(\Psi,\filtbk)$ is a tight frame for $l_2(\gph)$ for $J_1= 1,\ldots,J$, we obtain by polarization identity,
\begin{equation}\label{eq:f:UFS}
\begin{aligned}
  f = \cfrpra{J_1}(f) + \sum_{j=J_1}^{J}\sum_{r=1}^{n} \cfrprb{j}(f)
    = \cfrpra{J_1+1}(f) + \sum_{j=J_1+1}^{J}\sum_{r=1}^{n} \cfrprb{j}(f)
\end{aligned}
\end{equation}
for all $f\in l_2(\gph)$ and for all $J_1= 1,\ldots,J$.
Thus, for $J_1=1,\ldots,J-1$, 
\begin{equation}\label{eq:ufr.pr.J.J1}
 \cfrpra{J_1+1}(f) =  \cfrpra{J_1}(f) + \sum_{r=1}^{n} \cfrprb{J_1}(f),
\end{equation}
which is \eqref{thmeq:2scale1}. Moreover, when $J_1=J$, \eqref{eq:f:UFS} gives \eqref{thmeq:normalization1}.  Consequently, (i)$\Longrightarrow$(ii).
Conversely,  recursively using \eqref{eq:ufr.pr.J.J1} gives
\begin{equation}\label{eq:cfrpra.m1}
  \cfrpra{m+1}(f) = \cfrpra{J_1}(f) + \sum_{j=J_1}^{m}\sum_{r=1}^{n} \cfrprb{j}(f)
\end{equation}
for all $J_1\le m\le J-1$. Taking $m=J-1$ together with \eqref{thmeq:normalization1}, we deduce \eqref{eq:f:UFS},
which is equivalent to \eqref{eq:f:UFS0}. Thus, (ii)$\Longrightarrow$(i).

(ii)$\Longleftrightarrow$(iii). The equivalence between (ii) and (iii) simply follows from the polarization identity.

(ii)$\Longleftrightarrow$(iv). By the orthonormality of $\eigfm$,
\begin{equation*}\label{eq:cfr.coeff}
    \ipG{f,\cfra} = \sum_{\ell=1}^{\NV} \conj{\FT{\scala}\left(\frac{\eigvm}{2^{j}}\right)}\Fcoem{f}\:\eigfm(\uG), \quad
    \ipG{f,\cfrb{r}} = \sum_{\ell=1}^{\NV} \conj{\FT{\scalb^{(r)}}\left(\frac{\eigvm}{2^{j}}\right)}\Fcoem{f}\:\eigfm(\uG),
\end{equation*}
where $\FT{f}_\ell=\ipG{f,\eigfm}$ is the Fourier coefficient of $f$ with respect to $\eigfm$.
This together with \eqref{eqs:proj.ufr} and \eqref{defn:ufra:ufrb} gives, for $j\ge 1$ and $r=1,\dots,n$, the Fourier coefficients for the projections $\cfrpra{j}(f)$ and $\cfrprb{j}(f)$:
\begin{equation}\label{eq:Fcoe.cfrpr}
    \Fcoem{\left(\cfrpra{j}(f)\right)}
    = \left|\FT{\scala}\left(\frac{\eigvm}{2^{j}}\right)\right|^{2} \Fcoem{f},\quad
    \Fcoem{\left(\cfrprb{j}(f)\right)}
    = \left|\FT{\scalb^{(r)}}\left(\frac{\eigvm}{2^{j}}\right)\right|^{2} \Fcoem{f},\quad  \forall\ell=1,\ldots,\NV,
\end{equation}
which implies that \eqref{thmeq:normalization1} and \eqref{thmeq:2scale1} are equivalent to \eqref{thmeq:nrm:alpha:beta} and \eqref{thmeq:2scale:alpha:beta} respectively.
Thus, (ii)$\Longleftrightarrow$(iv).

(iv)$\Longleftrightarrow$(v).  By the relation in \eqref{eq:refinement}, it can be deduced that for $\ell=1,\dots,N$ and $j\ge 1$,
\begin{align*}
     \left|\FT{\scala}\left(\frac{\eigvm}{2^{j}}\right)\right|^{2} + \sum_{r=1}^{n}\left|\FT{\scalb^{(r)}}\left(\frac{\eigvm}{2^{j}}\right)\right|^{2}
    = \left(\left|\FT{\maska}\left(\frac{\eigvm}{2^{j+1}}\right)\right|^{2} + \sum_{r=1}^{n}\left|\FT{\maskb}\left(\frac{\eigvm}{2^{j+1}}\right)\right|^{2}\right)\left|\FT{\scala}\left(\frac{\eigvm}{2^{j+1}}\right)\right|^{2}.
\end{align*}
This shows that \eqref{thmeq:2scale:alpha:beta} is equivalent to \eqref{thmeq:2scale:masks}. Therefore, (iv)$\Longleftrightarrow$(v).
\end{proof}

\section{Framelet Transforms}
With framelet system $\ufrsys[J_1]^{J}(\Psi,\filtbk;\gph), J_1=1,\ldots, J$ introduced in Section~\ref{append:ufg}, we define the \emph{framelet coefficients} for a function $f$ on the graph $\gph$ by
\begin{equation}\label{eq:framelet coeff}
    \fracoev[0] = \bigl\{\ipG{f,\cfra[0,\uG]}\bigr\}_{p\in V}\mbox{~~and~~} \frbcoev{r} = \bigl\{\ipG{f,\cfrb[j,\uG]{r}}\bigr\}_{p\in V}, \quad j=0,\dots,J,\;\; r=1,\dots,n,
\end{equation}
where $\fracoev[0]$ and $\frbcoev{n}$ are the \emph{low-pass} and \emph{high-pass} coefficients for $f$. The \emph{framelet transforms} are the mapping between the graph signal $f$ and its framelet coefficients $\{\fracoev[0];\frbcoev[0]{n},\dots,\frbcoev[J]{n}\}$. The \eqref{eq:framelet coeff} can be written as the matrix-vector form, as follows.
For the eigenpair $\{(\eigvm,\eigfm)\}_{\ell=1}^{N}$ for the graph Laplacian $\gL$,
let $U = [\eigfm[1],\dots,\eigfm[N]]$ be the square matrix (of size $N\times N$) of eigenvectors, and $\Lambda = {\rm diag}(\eigvm[1],\dots,\eigvm[N])$ be the diagonal of eigenvalues. We then have
\begin{equation*}
    \FT{\scala}\left(\frac{\Lambda}{2^{j+1}}\right)={\rm diag}\left(\FT{\scala}\left(\frac{\eigvm[1]}{2^{j+1}}\right),\dots,\FT{\scala}\left(\frac{\eigvm[N]}{2^{j+1}}\right)\right),\quad 
    \FT{\scalb^{(r)}}\left(\frac{\Lambda}{2^{j+1}}\right)={\rm diag}\left(\FT{\scalb^{(r)}}\left(\frac{\eigvm[1]}{2^{j+1}}\right),\dots,\FT{\scalb^{(r)}}\left(\frac{\eigvm[N]}{2^{j+1}}\right)\right)
\end{equation*}
the filtered diagonal matrices with low-pass and high-pass filters. Then, the framelet coefficients have the following representation.
\begin{proposition} With the notation given above,
\begin{equation}\label{eq:coeff mat}
    \fracoev[0] = U \FT{\scala}\left(\frac{\Lambda}{2}\right) U^{^{\top}}f\mbox{~~and~~}
     \frbcoev{r} = U \FT{\scalb^{(r)}}\left(\frac{\Lambda}{2^{j+1}}\right) U^{^{\top}}f\quad\forall j=0,\dots,J,\; r=1,\dots,n.
\end{equation}
\end{proposition}

\subsection{Decomposition and Reconstruction}
We call estimating coefficients from $f$ \emph{framelet decomposition}, and its inverse \emph{framelet reconstruction}. 
In our construction for the framelets and framelet transforms, the framelet decomposition and reconstruction are \emph{invertible}.
The decomposition and reconstruction can be achieved efficiently via a \emph{filter bank} $\filtbk=\{\maska; \maskb[1],\ldots,\maskb[n]\}$. We use the filter bank of Haar-type in \citet{dong2017sparse}, by which the decomposition and reconstruction can be implemented recursively in a fast algorithm. 
By repeated use of the refinement equation \eqref{eq:refinement} of the filter, we arrive at the following transforms, which pave the way for efficiently computing framelet coefficients in \eqref{eq:coeff mat}. For $r=1,\dots,n$ and $j=1,\dots,J$, we define operators for $(r,j)\in \{(1,1),\ldots,(1,J),\ldots(n,1),\ldots,(n,J)\}\cup \{(0,J)\}$ by, for $f\in l_2(\gph)$,
\begin{equation}\label{eq:Wflat}
    \begin{aligned}
    \analOp_{r,1}^{\flat}f &= U \FT{\maskb}\left(2^{-K}\Lambda\right)U^{\top}f, \\
    \analOp_{r,j}^{\flat}f &= U \FT{\maskb}\left(2^{K+j-1}\Lambda\right)
    \FT{\maska}\left(2^{K+j-2}\Lambda\right)\dots\FT{\maska}\left(2^{-K}\Lambda\right)U^{\top}f,\quad\forall j\geq2,
\end{aligned}
\end{equation}
where as mentioned before, $K$ is the real value such that the graph Laplacian's biggest eigenvalue $\eigvm[{\rm max}]\leq 2^K\pi$.
With the transforms $\analOp_{r,j}^{\flat}$ in \eqref{eq:Wflat}, we can write the the decomposition and reconstruction explicitly, as follows. 
\begin{theorem}[Framelet Decomposition and Reconstruction]\label{thm:decomp reconstr} 
The framelet decomposition can be achieved via filter bank $\filtbk$ recursively: for $r=1,\dots,n$ and $j=1,\dots,J$, 
\begin{equation}\label{eq:framelet decomp}
    \fracoev[0] = \analOp_{0,J}^{\flat}f
    \mbox{~~and~~} 
    \frbcoev{r} = \analOp_{r,j}^{\flat}f.
\end{equation}
The reconstruction for a set of coefficients $\{\fracoev[0]\}\cup\{\frbcoev{1},\dots,\frbcoev{n}\}_{j=1}^{J}$ on $\gph$ can be computed by
\begin{equation}\label{eq:framelet reconstr}
    f_J = \analOp_{0,J}^{\flat,\star}\:\fracoev[0] + \sum_{j=1}^{J}\sum_{r=1}^{n}\analOp_{r,j}^{\flat,\star}\:\frbcoev{r},
\end{equation}
where the $\star$ indicates the conjugate transpose of the associated matrix.
The decomposition and reconstruction in \eqref{eq:framelet decomp} and \eqref{eq:framelet reconstr} are invertible, that is, $f_J=f$.
\end{theorem}
\begin{proof} We prove the invertibility between \eqref{eq:framelet decomp} and \eqref{eq:framelet reconstr}. As mentioned above, the framelet coefficients in \eqref{eq:framelet decomp} are equivalent with \eqref{eq:coeff mat} and then with the original definition in \eqref{eq:framelet coeff}, due to that the scaling and filter functions satisfy the refinement equation \eqref{eq:refinement}. 
By the equivalence between Theorem~\ref{thm:UFS}(v) and (ii), using the framelet coefficients in \eqref{eq:framelet decomp} for \eqref{eq:framelet reconstr}, we thus obtain $f_J=f$.
\end{proof}

Note that we count the level index $j$ in \eqref{eq:tensor reconstr} from $1$, which is in essence the same as \eqref{eq:f:UFS0} of Theorem~\ref{eq:framelet reconstr}, and Theorem~\ref{thm:UFS} below.

\subsection{Tensor-based $\gph$-framelet Transforms}
Due to the computational difficulty of eigendecomposition for the large-scale graph Laplacian matrix, the decomposition and reconstruction in Theorem~\ref{thm:decomp reconstr} cannot be directly computed in an efficient way.
To fast evaluate them, we apply the approximation by Chebyshev polynomials $\mathcal{T}_{0},\dots,\mathcal{T}_{n}$ of a fixed degree $t$, for the filter $\maska\approx \mathcal{T}_{0}$ and $\maskb\approx \mathcal{T}_{r}$. Here $t$ is a sufficiently large integer such that the Chebyshev polynomial approximation is of high precision. 
Then, the \eqref{eq:Wflat} can be approximated by
\begin{equation}\label{eq:Wflat 1}
    \begin{aligned}
    \analOp_{r,1}^{\flat}f &= U \FT{\maskb}\left(2^{-K}\Lambda\right)U^{\top}f
    \approx U \mathcal{T}_{r}\left(2^{-K}\Lambda\right)U^{\top}f, \\
    \analOp_{r,j}^{\flat}f &= U \FT{\maskb}\left(2^{K+j-1}\Lambda\right)
    \FT{\maska}\left(2^{K+j-2}\Lambda\right)\dots\FT{\maska}\left(2^{-K}\Lambda\right)U^{\top}f\\
    &\approx U \mathcal{T}_{r}\left(2^{K+j-1}\Lambda\right)
    \mathcal{T}_{0}\left(2^{K+j-2}\Lambda\right)\dots\mathcal{T}_{0}\left(2^{-K}\Lambda\right)U^{\top}f,
    \quad\forall j\geq2.
\end{aligned}
\end{equation}
By the property of the polynomial of matrix and the eigendecomposition of the graph Laplacian $U\Lambda U^{\top} = \gL$, the \eqref{eq:Wflat 1} then becomes
\begin{equation}
    \begin{aligned}
    \analOp_{r,1}^{\flat}f 
    &\approx \mathcal{T}_{r}\left(2^{-K}\gL\right)f
    =:\analOp_{r,1}f, \\
    \analOp_{r,j}^{\flat}f 
    &\approx \mathcal{T}_{r}\left(2^{K+j-1}\gL\right)
    \mathcal{T}_{0}\left(2^{K+j-2}\gL\right)\dots\mathcal{T}_{0}\left(2^{-K}\gL\right)f =: \analOp_{r,j}f,
    \quad\forall j\geq2.
\end{aligned}
\end{equation}
With filters given, we can pre-compute the coefficients for the Chebyshev polynomial approximation up to degree $t$ by a quadrature rule. The framelet decomposition and reconstruction can be evaluated by the approximate transforms $\analOp_{r,j}$ analogously to Theorem~\ref{thm:decomp reconstr}, as follows.
\begin{theorem}[Framelet Transforms by Chebyshev Polynomial Approximation]\label{thm:tensor decomp reconstr} 
The framelet decomposition and reconstruction can be approximated by: for $r=1,\dots,n$ and $j=1,\dots,J$,
\begin{equation*}
    \fracoev[0] = \analOp_{0,J}f\mbox{~~and~~} \frbcoev{r} = \analOp_{r,j}f.
\end{equation*}
The reconstruction for a set of coefficients $\{\fracoev[0]\}\cup\{\frbcoev{1},\dots,\frbcoev{n}\}_{j=1}^{J}$ on $\gph$ can be evaluated by
\begin{equation}\label{eq:tensor reconstr}
    f_J = \analOp_{0,J}^{\star}\:\fracoev[0] + \sum_{r=1}^{n}\sum_{j=1}^{J}\analOp_{r,j}^{\star}\:\frbcoev{r},
\end{equation}
where the $\star$ indicates the conjugate transpose of the associated matrix.
\end{theorem}

As mentioned, with proper alignment of $\analOp_{r,j}$, we have a tensor-based evaluation for framelet transforms. Define the matrix (of size $(nJ+1)N\times N$)
\begin{equation}
    \analOp^{\natural} = \bigl[\analOp_{0,J},\analOp_{1,J},
    \dots,\analOp_{n,J},\dots,\analOp_{1,J},
    \dots,\analOp_{n,J}\bigr]^{\top}.
\end{equation}
Let $c$ be the concatenation of coefficients $\fracoev[0],\frbcoev[1]{1},\ldots,\frbcoev[1]{n},\ldots,\frbcoev[J]{1},\ldots,\frbcoev[J]{n}$ associated with a framelet system with $n$ high passes and graph with $N$ nodes. Then, $c$ is a column vector of length $(nJ+1)N$. The following corollary gives the decomposition and reconstruction for the tensor-based $\gph$-framelet transforms, as used in the main part and experiments.
\begin{corollary} 
The coefficients from framelet decomposition up to scale level $J$ is given by
\begin{equation*}
    c= \analOp^{\natural} f
\end{equation*}
and the framelet reconstruction is given by
\begin{equation*}
    f_J = (\analOp^{\natural})^{\top}c.
\end{equation*}
\end{corollary}

\section{Shrinkage for Wavelet Denoising and LASSO}
Wavelets play an important role in denoising. The typical model is to learn a function $f$ on $[0,1]$ from one dimensional noisy data
\begin{equation}\label{eq:noisy data}
    y_i = f(x_i) + \sigma \epsilon_i,\quad i=1,\dots,N,
\end{equation}
for given $\sigma$ and i.i.d. white noise $\epsilon_i$. \citet{donoho1995noising} gave a method of finding an approximate $f$ from \eqref{eq:noisy data} by 1D wavelets and soft-thresholding shrinkage.
Typical steps include: 
\begin{enumerate}
    \item Apply pyramid wavelet filtering \cite{cohen1993multiresolution} for the scaled input data $y_i/\sqrt{N}$, which then yields noisy wavelet coefficients up to scale level $J$: $w_{j,k}$, $j=0,\dots, J$, $k=0,\dots,2^j-1$.
    \item Use soft-threshold nonlinearity to the high-pass wavelets $w_{j,k}$, with threshold value $\sigma\sqrt{2\log (N)/N}$. This then gives an estimate $w_{j,k}^{\sharp}$ for the original wavelet coefficients.
    \item Reconstruct the signal by using inverse wavelet transforms for the shrinkaged coefficients.
\end{enumerate}
This shrinkage method is also used as an example of statistical multivariate estimation by \citet{efron1975data}. 

Moreover, shrinkage plays a pivotal role in LASSO that estimates the coefficients of regression
\begin{equation}\label{eq:lasso model}
    {\rm min}_{\alpha,\beta_j}\left\{\sum_{i=1}^N\left(y_i-\alpha-\sum_{ j}\beta_j x_{i,j}\right)\right\}\quad \mbox{subject to~} \sum_{j}|\beta_j|\leq t.
\end{equation}
The soft-thresholding provides a solution to \eqref{eq:lasso model}, and thus a regression selection for variables. The unbiased estimate of the risk or mean-squared error for the estimator $\widehat{\beta}$ is equivalent to the risk for the shrinkage wavelet denoising approximator in \citet{donoho1994ideal}.

\paragraph{Connection to Data Compression} With the above shrinkage for wavelet denoising problem \eqref{eq:noisy data}, we can remove the noise and also compress the wavelet representation which considerably reduces the number of coefficients. The inverse wavelet transform for thresholded coefficients compresses the $l_2$-energy of the signal into a few number of large wavelet coefficients but the Gaussian white noise in the signal or wavelet coefficients does not change their energy level. Thus, the large coefficients that contain main information of the original signal can be separated and distilled from the white noise. See \citet{donoho1992wavelet,Dong2013mra}.

\section{Comparison of Existing Spectral-based GNNs}
We compare the key characteristics of our proposed \textsc{UFGConv} with four classic spectral convolution models for a better understanding. The summary is provided in Table~\ref{tab:comparison}.

\begin{table*}[!htbp]
\caption{Brief summary for the existing spectral-based GNNs.}
\label{tab:comparison}
\scriptsize
\begin{center}
\begin{tabular}{llllc}
\toprule
\textbf{Model}  & \textbf{Graph convolution}    &   \textbf{Network filter} & \textbf{Computational strategy}    & \textbf{Multi-scale}\\
\midrule
\textsc{SpectralCNN} \cite{bruna2013spectral} & $y=U{\btheta'}U^{T}X$ &  ${\btheta'}=\mbox{Diag}(\btheta)$ &  No, $U$ is based on $L = U\Lambda U^{T}$   &  $\times$\\
\textsc{ChebyShev} \cite{DeBrVa2016}   &$y=U{\btheta'}U^{T}X$           &  ${\btheta'}=\sum_{k=0}^{K-1}\alpha_k\Lambda^{k}$ &  No, $y=\sum_{k=0}^{K-1}\alpha_kL^{k}$   &  $\times$\\
\textsc{GWNN} \cite{xu2019graph} & $y=\Psi_s{\btheta'}\Psi_s^{-1}X$ &  ${\btheta'}=\mbox{Diag}(\btheta)$ &  Yes, $\Psi_s$ is graph wavelet transform &  $\surd$\\
\textsc{HaNet} \cite{li2020fast} & $y=\Phi \btheta'\Phi^{T}x$ &  $\btheta'=\mbox{Diag}(\btheta)$ &  Yes, $\Phi$ is Haar transforms    &  $\surd$\\
\midrule
\textsc{UFGConv-R} & $y={\rm ReLU}\left(\synOp\left(\btheta'(\analOp X)\right)\right)$ &  $\btheta'=\mbox{Diag}(\btheta)$ &  $\analOp,\synOp$ are framelet transforms &  $\surd$\\
\textsc{UFGConv-S} & $y=\synOp\left({\rm Shrinkage}\bigl(\btheta' (\analOp X)\bigr)\right)$ &  $\btheta'=\mbox{Diag}(\btheta)$ &  $\analOp,\synOp$ are framelet transforms &  $\surd$\\[1mm]
\bottomrule\\[-2.5mm]
\multicolumn{5}{l}{$\dagger$ $G_s=\mbox{diag}\{e^{s\lambda_1},\ldots,e^{s\lambda_N}\}$ and $\Psi_s$ can be approximated fast via Chebyshev polynomial \cite{HaVaRe2011}}
\end{tabular}
\end{center}
\end{table*}

\section{Experiments}
\label{exp_setting}
In this section, we provide more details for each experiment conducted in this paper. Also, we attach more experiment results regarding the robustness analysis of our proposed \textsc{UFGConv} on \textbf{Citeseer} and \textbf{Pubmed} in Section~\ref{sec:more:robust}, and the visualizations of the shrinkage effect on \textbf{Cora} in Section~\ref{sec:shrinkage:cora}. For all the experiments, we initialize parameters of our model \textsc{UFGConv} according to: \emph{zeros} for the bias term; \emph{xavier uniform} \citep{glorot2010understanding} for the weight matrix $W$; and \emph{uniform distribution} $\mathcal{U}(0.9, 1.1)$ for the network filter $\btheta$.

\subsection{Node Classification}
\label{sec:sm:node_classification}
The first experiment of node classification, which corresponds to the results in Table~\ref{tab:citation_node}, is performed on the benchmark citation networks. For both \textsc{UFGConv-S} and \textsc{UFGConv-R}, we use default values for the scale level $(J=2)$ and the dilation $(\textnormal{base}=2)$. The rest of the hyperparameters are tuned by using a grid search with the following search spaces listed in the ``First experiment'' column in Table~\ref{tab:searchSpace node}.

\begin{table}[t]
\caption{Hyperparameter searching space for node classification.}
\label{tab:searchSpace node}
\begin{center}
\begin{tabular}{lrr}
\toprule
\textbf{Hyperparameters} & \textbf{First experiment} & \textbf{Second experiment}\\\midrule
Learning rate & 5e-2, 1e-2, 5e-3 & 5e-3, 1e-3, 5e-4\\
Weight decay ($L_2$) \; & 5e-2, 1e-2, 5e-3 & 5e-3, 1e-3, 5e-4\\
Hidden size & 16, 32, 64 & 128, 256\\
Dropout ratio & 0.5, 0.6, 0.7 & 0.5, 0.6, 0.7 \\
Scale level & - & 2,3\\
Dilation & - & 1.5, 2.0, 2.5\\\bottomrule
\end{tabular}
\end{center}
\end{table}

The second node classification task whose results shown in Table~\ref{tab:ogbn_arxiv} is conducted on \textbf{ogbn-arxiv} from \textbf{OGB} \citep{hu2020open}. It demonstrates the superior performance of our proposed model \textsc{UFGConv} on large-scale graph-structured data. We select a variety of existing models with their publicly available prediction performances as the baselines. A full list of the leaderboard can be found in the OGB website\footnote{\url{https://ogb.stanford.edu/docs/leader_nodeprop/\#ogbn-arxiv}}. For a fair comparison, our model aligns with the model structure of \textsc{GCN} \citep{KiWe2017} and \textsc{GraphSAGE}'s \citep{hamilton2017inductive} that produces the scores on \textbf{ogbn-arxiv} leaderboard: three convolutional layers with one dropout layer inserted between every two consecutive convolutional layers. Moreover, a batch normalization \citep{ioffe2015batch} is applied after each convolutional layer (before the ReLU activation function for \textsc{UFGConv-R} or after the reconstruction process for \textsc{UFGConv-S}). Again, a grid search is performed for fine tuning the hyperparameters in the search spaces listed in the ``Second experiment'' column in Table~\ref{tab:searchSpace node}.

The results of the trade-off analysis with \textbf{ogbn-arxiv} in Figure~\ref{Fig:tradeoff} correspond to multiple runs of \textsc{UFGConv-S} with different shrinkage thresholds $\sigma=[1, 3, 5, 7, 9]$. Other hyperparameters are fixed at: $0.001$ for learning rate, $0.001$ for weight decay, $2.0$ for dilation, $0.5$ for dropout, and $3$ for scale level.

Table~\ref{tab:stats:node_classification} documents some descriptive statistics of the datasets used for the node classification tasks.

\begin{table}[th]
\caption{Summary of the datasets for node classification tasks.}
\label{tab:stats:node_classification}
\begin{center}
\begin{tabular}{lcccc}
\toprule
 & \textbf{Cora} & \textbf{Citeseer} & \textbf{Pubmed} & \textbf{ogbn-arxiv}\\
\midrule
\# Nodes & $2,708$ & $3,327$ & $19,717$ & $169,343$ \\
\# Edges & $5,429$ & $4,732$ & $44,338$ & $1,166,243$ \\
\# Features & $1,433$ & $3,703$ & $500$ & $128$ \\
\# Classes & $7$ & $6$ & $3$ & $40$ \\
\# Training Nodes & $140$ & $120$ & $60$ & $90,941$\\
\# Validation Nodes & $500$ & $500$ & $500$ & $29,799$ \\
\# Test Nodes & $1,000$ & $1,000$ & $1,000$ & $48,603$\\
Label Rate & $0.052$ & $0.036$ & $0.003$ & $0.537$ \\
\bottomrule
\end{tabular}
\end{center}
\end{table}

\subsection{Robustness Analysis on \textbf{Citeseer} and \textbf{Pubmed}}
\label{sec:more:robust}
This section supplements the perturbation analysis in Section~\ref{sec:robustness}. In this experiment, we compare our \textsc{UFGConv} (both shrinkage and ReLU models) against \textsc{GCN} \citep{KiWe2017} and \textsc{GAT} \citep{velivckovic2017graph} on the perturbed citation networks. All the models consist of two convolutional layers with a dropout layer inserted in between. The hyperparameters are, unless further specified, tuned in the same way with the same searching spaces as illustrated in the first experiment of Section~\ref{sec:sm:node_classification}. The shrinkage threshold $\sigma$ is searched over $\{0.05, 0.10, 0.15\}$. For node attribute perturbation on \textbf{Pubmed}, the dropout ratio is searched from the set $\{0, 0.5\}$; the learning rate is searched from a larger set of $\{\hbox{1e-2, 5e-3, 1e-3, 5e-4}\}$.

The structure noise on all three datasets are generated analogously. We define the noise ratio as the number of connected node pairs in the new graph divided by the number of connected node pairs in the original graph. The ratio at 1 represents the undistorted graph. For node attribute perturbation, noise defined on \textbf{Cora} and \textbf{Citeseer} follows Bernoulli distribution, where we change a small portion of node features from 0 (or 1) to 1 (or 0). A selection of noise ratios in the range from $0.1$ to $2$ with a step size $0.1$ are considered.
On \textbf{Pubmed}, as the raw data are real values, we add a Gaussian noise with zero mean and standard deviation $\sigma_s$ on the nodes, where $\sigma_s$ is the noise ratio in the plot. The noise ratio ranges from $0.01$ to $0.15$ with step size $0.01$ union $\{0.005,0.015,0.025\}$.

We refer the reader to Figure~\ref{fig:attack} in the main text for the experimental result on \textbf{Cora}, and Figures~\ref{Fig:attackCiteseer}-\ref{fig:attackPubmed} for \textbf{Citeseer} and \textbf{Pubmed}.

\begin{figure}[!htb]
    \begin{minipage}{0.48\textwidth}
    \centering
    \includegraphics[width=0.9\linewidth]{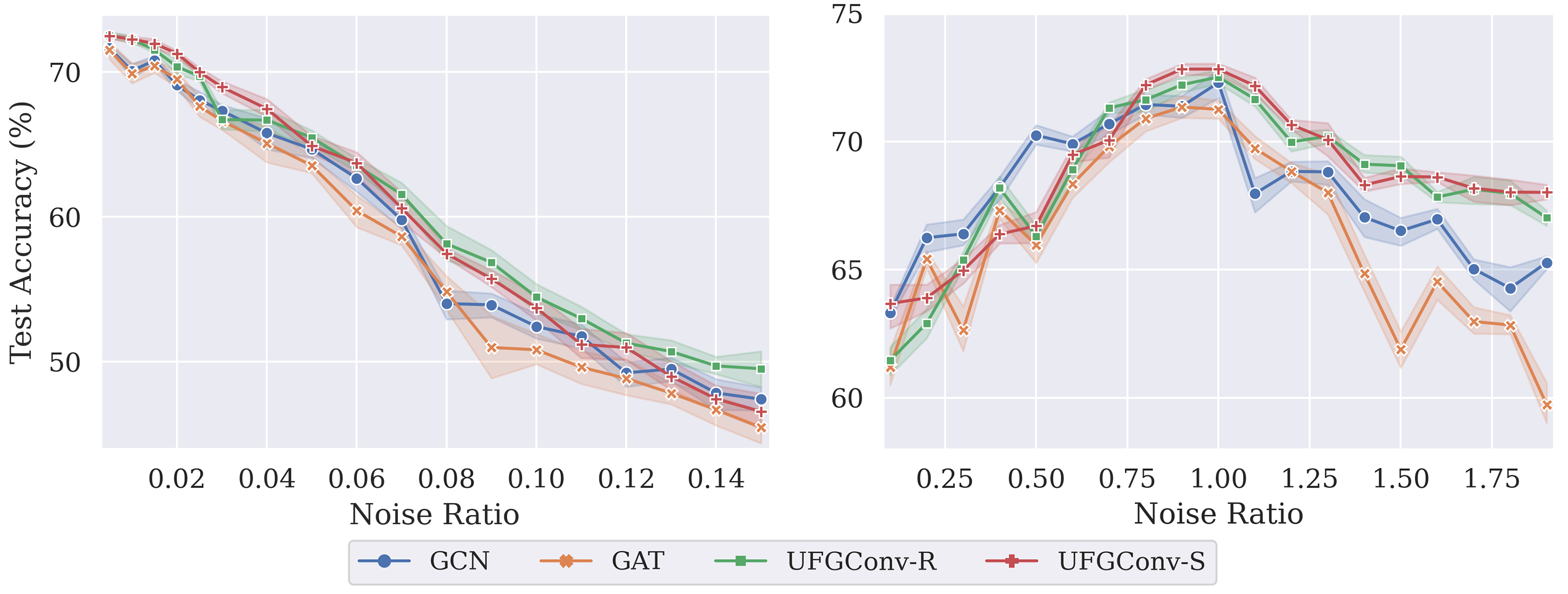}
    \vspace{-3mm}
    \caption{Node attribute (left) and graph structure (right) perturbation analysis on \textbf{Citeseer}.}
    \label{Fig:attackCiteseer}
    \end{minipage}
    \hspace{0.5cm}
    \begin{minipage}{0.48\textwidth}
    \centering
    \includegraphics[width=0.9\linewidth]{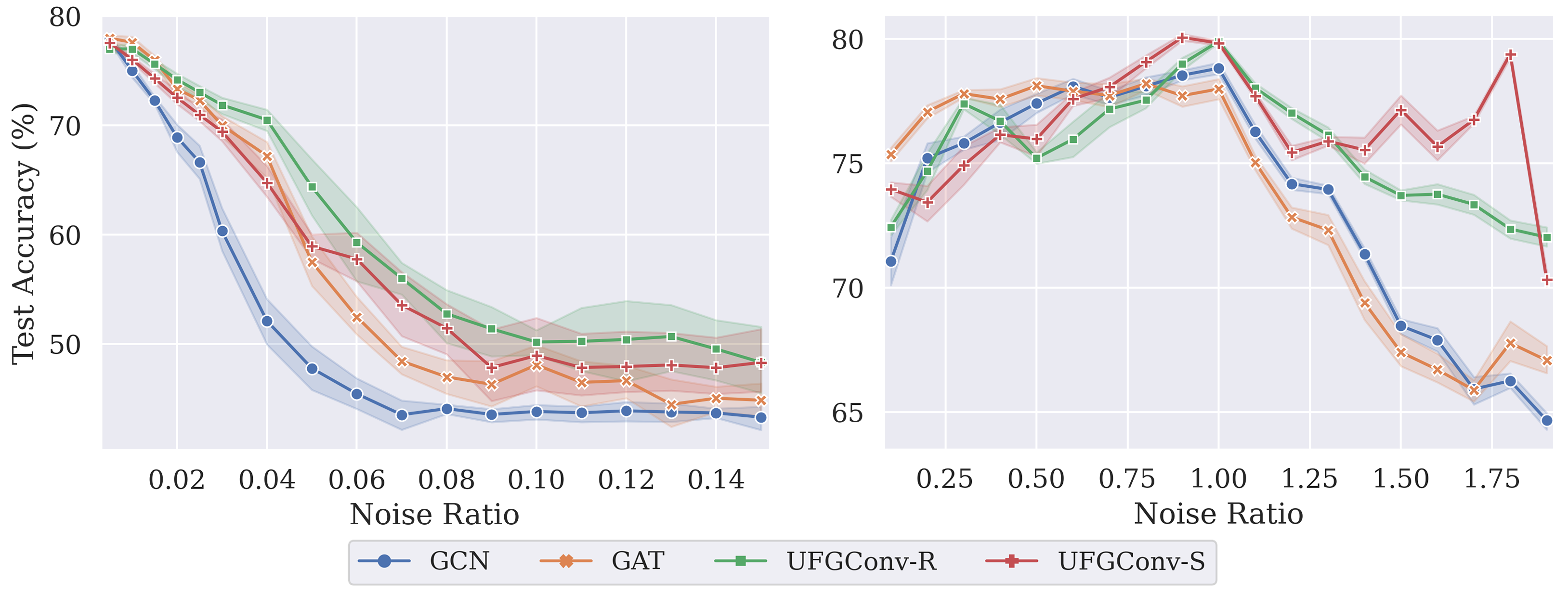}
    \vspace{-3mm}
    \caption{Node attribute (left) and graph structure (right) perturbation analysis on \textbf{Pubmed}.}
    \label{fig:attackPubmed}
    \end{minipage}
\end{figure}

The performance of our proposed \textsc{UFGConv} (with either shrinkage or ReLU activation) on \textbf{Citeseer} and \textbf{Pubmed} aligns with the superior performance on \textbf{Cora}, which indicates that our models are more robust than the baseline methods, especially on datasets added with more noise.

\subsection{Graph Classification and Regression}
We supplement basic information of the datasets for graph property prediction tasks in Section~\ref{sec:experiment ufgpool}. We report the descriptive statistics in Table~\ref{table:graph descriptive} for the six benchmark datasets used in this experiment. The `R' in the bracket of the last row of \textbf{QM7} represents regression task.

\begin{table}[th]
\caption{Summary of the datasets for the graph property prediction tasks.}
\label{table:graph descriptive}
\begin{center}
\begin{tabular}{lcccccc}
\toprule
\textbf{Datasets} & \textbf{PROTEINS} & \textbf{Mutagenicity} & \textbf{D\&D} & \textbf{NCI1} & \textbf{ogbg-molhiv} & \textbf{QM7} \\ 
\midrule
\# Graphs    & $1,113$ & $4,337$ & $1,178$ & $4,110$ & $41,127$ & $7,165$ \\
Min \# Nodes & $4$     & $4$     & $30$    & $3$     & $2$         & $4$     \\
Max \# Nodes & $620$   & $417$   & $5,748$ & $111$   & $222$    & $23$    \\
Avg \# Nodes & $39$    & $30$    & $284$   & $30$    & $26$     & $15$  \\
Avg \# Edges & $73$    & $31$    & $716$   & $32$    & $28$     & $123$ \\
\# Features  & $3$     & $14$    & $89$    & $37$    & $9$      & $0$     \\
\# Classes   & $2$     & $2$     & $2$     & $2$     & $2$      & $1$ (R)  \\
\bottomrule
\end{tabular}
\end{center}
\end{table}
\vspace{-5mm}
\begin{table}[th]
\caption{Hyperparameters for graph property prediction and sensitivity analysis.}
\begin{center}
\label{tab:searchSpace graph sensitivity}
\begin{tabular}{lrr}
\toprule
\textbf{Hyperparameters} & \textbf{Graph classify and regress} & \textbf{Sensitivity analysis}\\\midrule
Learning rate & 5e-3, 1e-3, 5e-4 & 5e-2, 1e-2, 5e-3\\
Weight decay ($L_2$) \; & 5e-3, 1e-3, 5e-4 & 5e-2 1e-2, 5e-3\\
Hidden size & 16, 32, 64 & 16, 32, 64\\
Dropout ratio & 0, 0.5 & 0.6, 0.7\\\bottomrule
\end{tabular}
\end{center}
\end{table}

The model architecture is set to 2 $\textsc{GCN}$ convolutional layers for \textbf{TU Dataset}s or 4 $\textsc{GIN}$ convolutional layers for \textbf{OGB}, which is consistent with the general \textbf{OGB} framework. We summarize the searching spaces of the key hyperparameters in Table~\ref{tab:searchSpace graph sensitivity}.
Other hyperparameters, if not specified, are set to default values. 

\subsection{Sensitivity Analysis}
The experiments documented in Section~\ref{sec:sensitity} analyze the sensitivity of our proposed \textsc{UFGConv} on the hyperparameters \emph{dilation} and \emph{scale level}. For each dataset (\textbf{Cora}, \textbf{Citeseer} or \textbf{Pubmed}), we fix the optimal hidden size to $16$, $32$ and $64$; the scale level to $2$, $2$ and $3$ (for dilation analysis); the dilation to $2$, $2$ and $2$ (for scale level analysis), respectively. We use shrinkage threshold $\sigma=1$ for \textsc{UFGConv-S}. We examine the sensitivity on dilation ranging from $1.25$ to $4$ with step $0.25$, and on scale level ranging from $1$ to $8$ with step $1$. The rest of the hyperparameters are searched over the spaces listed in Table~\ref{tab:searchSpace graph sensitivity}.

\subsection{Computational Complexity}
In Section~\ref{sec:mra}, we discuss the theoretical time and space complexities of the tensor-based framelet transforms which are the key procedures of our proposed \textsc{UFGConv}. In this section, we empirically examine the time complexity of our models against some classic graph convolutions (\textsc{GCN} \citep{KiWe2017}, \textsc{GAT} \citep{velivckovic2017graph} with $8$ attention heads, and \textsc{ChebyShev} \citep{defferrard2016convolutional} with polynomial degree $3$) on the $12$ random Erd\"{o}s-R\'{e}nyi graphs with different node sizes. In this experiment, we only time one forward propagation of each method and then repeat this procedure for $1,000$ times. The experiment is conducted on an NVIDIA\textsuperscript{\textregistered} Tesla V100 GPU card with 5,120 CUDA cores and 16GB HBM2 mounted on an HPC cluster. 

\begin{figure*}[!htb]
   \begin{minipage}{0.53\textwidth}
    \centering
    \includegraphics[width=\linewidth]{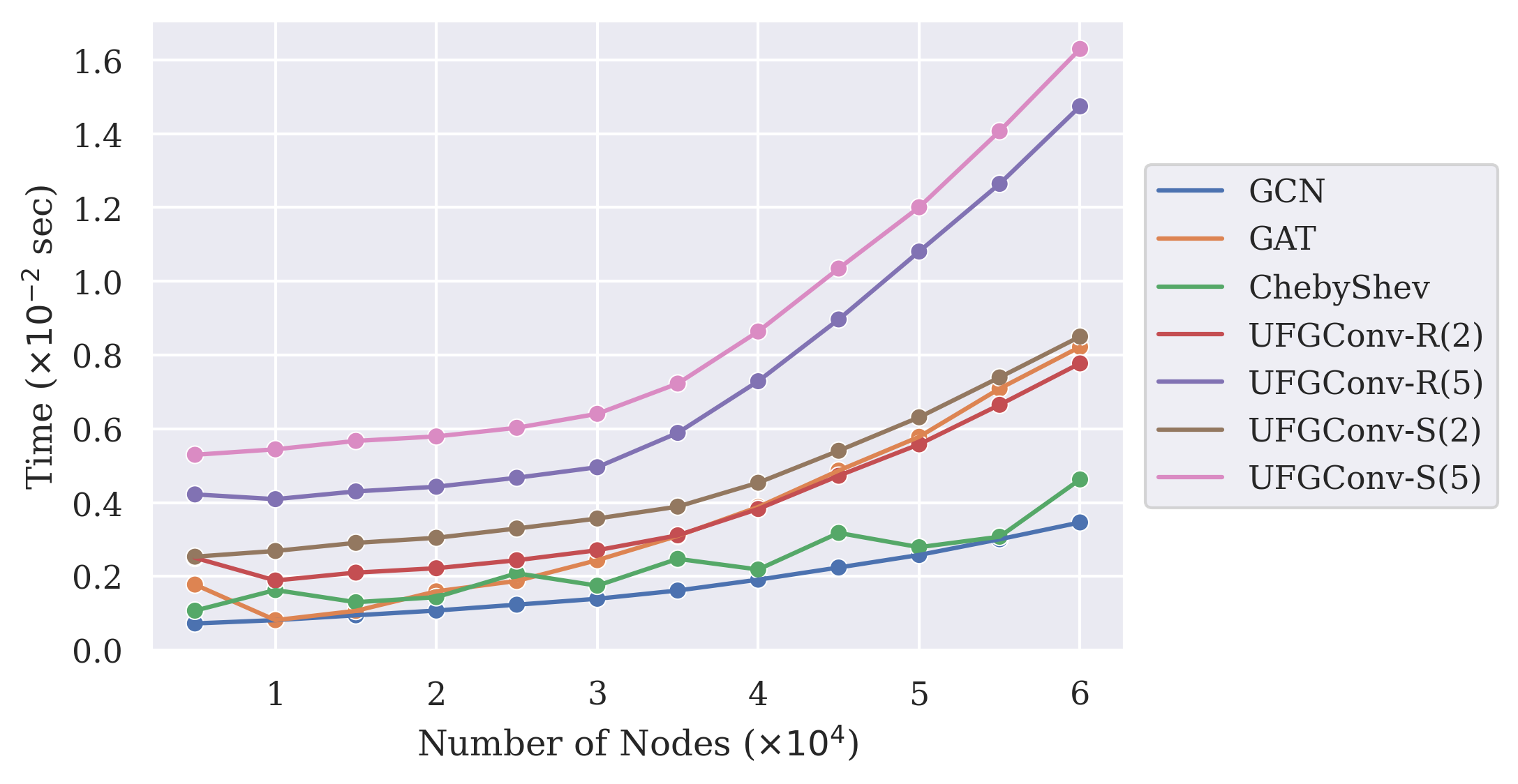}
    \vspace{-3mm}
    \caption{Computational cost (in seconds) against the node sizes of $12$ randomly generated graphs. The result is the average of $1,000$ runs. The value reported in the brackets after the name of our model in the legend indicates the corresponding scale level $J$.}\label{fig:computation}
   \end{minipage}\hfill
   \begin{minipage}{0.43\textwidth}
    \centering
    \includegraphics[width=0.8\linewidth]{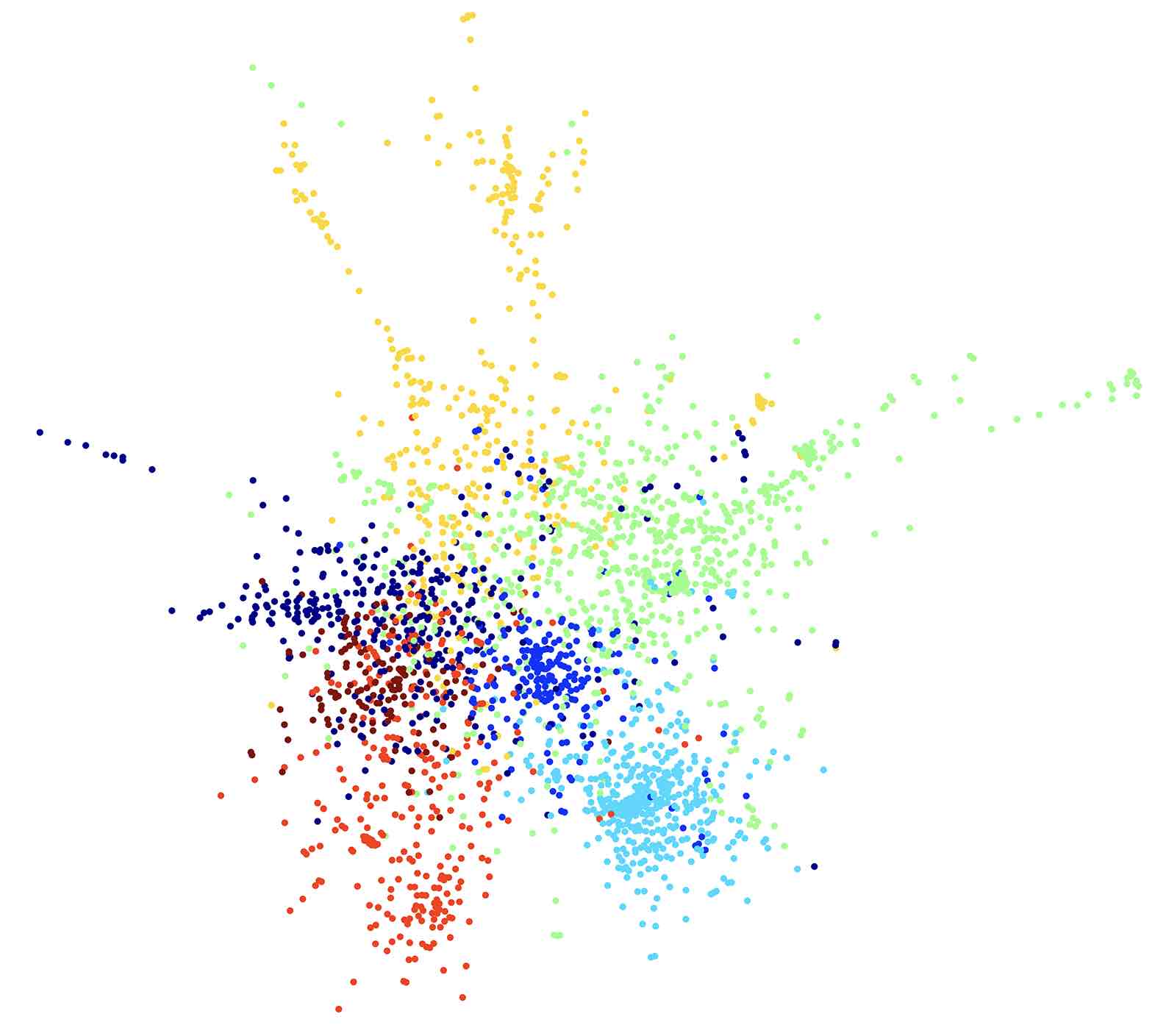}
    \caption{Visualization of \textsc{UFGConv-S} on \textbf{Cora}, with each color indicates one of $7$ classes that nodes come from.}
    \label{fig:cora7classes}
   \end{minipage}
\end{figure*}


From Figure~\ref{fig:computation}, we can observe that \textsc{UFGConv-S} has a slightly higher computational cost than its ReLU counterparts \textsc{UFGConv-R} under the same scale level, due to the extra steps for calculating the shrinkage threshold and trimming the framelet coefficients. However, the sensitivity analysis in Section~\ref{sec:sensitity} suggests that a small scale level is sufficient to achieve a better performance. Both of our proposed models (when scale level set to $2$) have a similar computational complexity as \textsc{GAT} \citep{velivckovic2017graph} (with $8$ attention heads). As the node size of a graph increases, the speed of \textsc{GAT} degrades rapidly and our models are faster than \textsc{GAT} when the graph has more than $4\times 10^4$ nodes.

\subsection{Visualization for \textsc{UFGConv} with Shrinkage on \textbf{Cora}}
\label{sec:shrinkage:cora}
In this section, we visually demonstrate the effectiveness of the shrinkage activation on \textbf{Cora}. A complete graph of all seven classes in \textbf{Cora} is displayed in Figure~\ref{fig:cora7classes}. We adopt \textsc{UFGConv-S} on \textbf{Cora} with a compression ratio of $47.7\%$ from Table~\ref{tab:citation_node} in the main text. 


The next seven figures visualize the framelet coefficients for all seven classes. In particular, we compare on each filter the coefficients at initialization, after two convolutional layers before shrinkage activation, and after shrinkage activation. All the values are retrieved under the evaluation mode, which means we leave the effect of random dropout. Also, the compressed results are only reported in the two high-pass filters, since the shrinkage activation does not apply on the low pass. 

For all the seven sets of results, a vertical comparison between the initialization and the other two columns indicate clearly that all the three passes learn the coefficients according to the true class information. Nodes with respect to the corresponding classes usually have higher absolute values that are more distinct from $0$. Also, coefficients on both high-passes compressed a critical part of coefficients after shrinkage (in green). The majority of leftovers, except for those from the highlighted class, have close-to-zero coefficients.

For a horizontal comparison, low-pass coefficients usually have higher but less distinctive values, while high-pass coefficients are more concentrated on the detailed information with respect to the individual classes.

\begin{figure*}[t]
    \centering
    \begin{annotate}{\includegraphics[width=\linewidth]{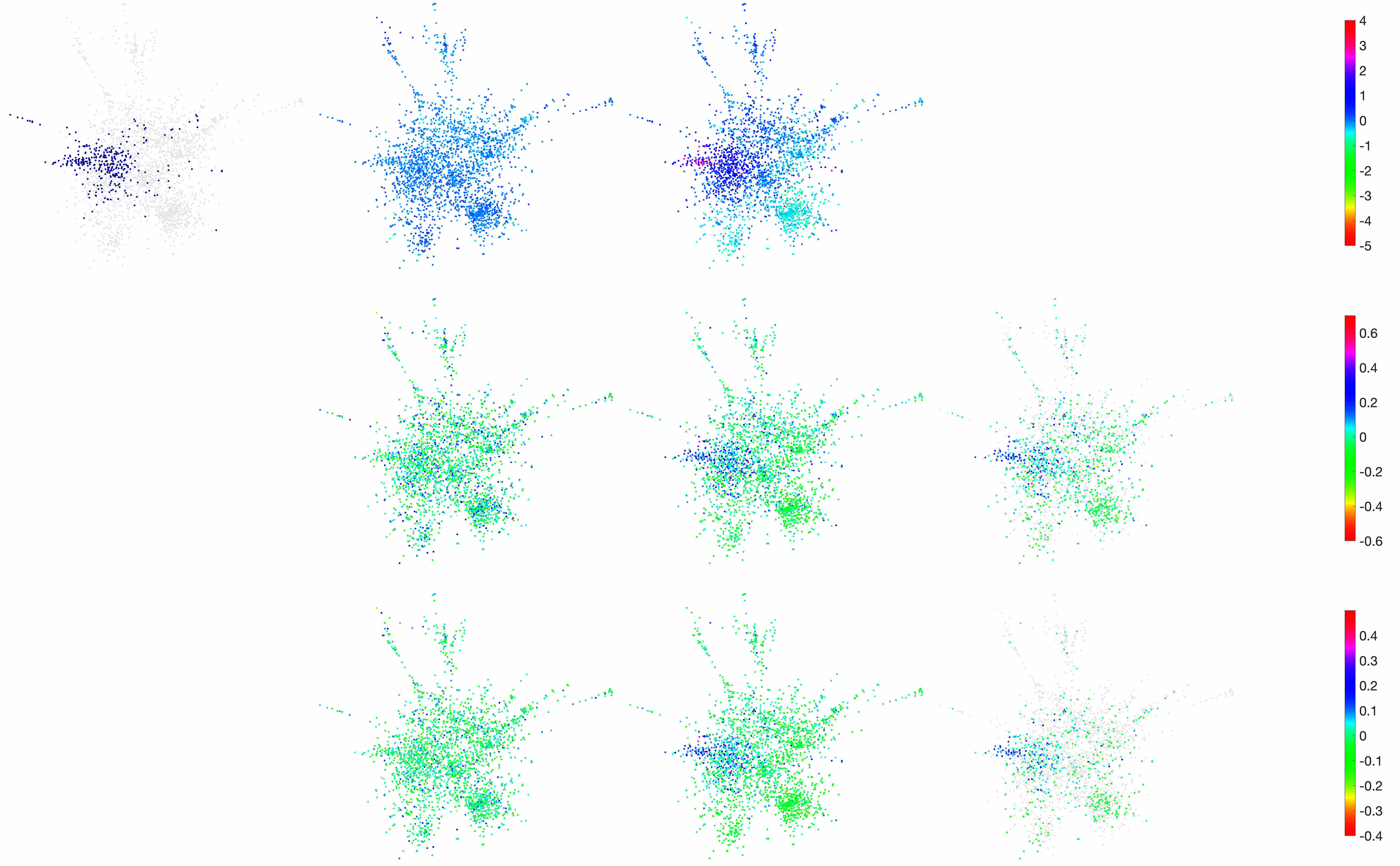}}{1}
    \note{-6.7,1.8}{Ground Truth}
    \note{-2.8,5.5}{Initialization}
    \note{1,5.5}{Before Shrinkage}
    \note{4.8,5.5}{After Shrinkage}
    \draw[ ] (-5,3.6,0) node [rotate=90] {\tiny{Low-Pass}};
    \draw[ ] (-5,0,0) node [rotate=90] {\tiny{High-Pass 1}};
    \draw[ ] (-5,-3.6,0) node [rotate=90] {\tiny{High-Pass 2}};
    \end{annotate}
    \vspace{-7mm}
    \caption{Framelet coefficients on \textbf{Cora}, Class $1$.}
    \vspace{2mm}
    \begin{annotate}{\includegraphics[width=\linewidth]{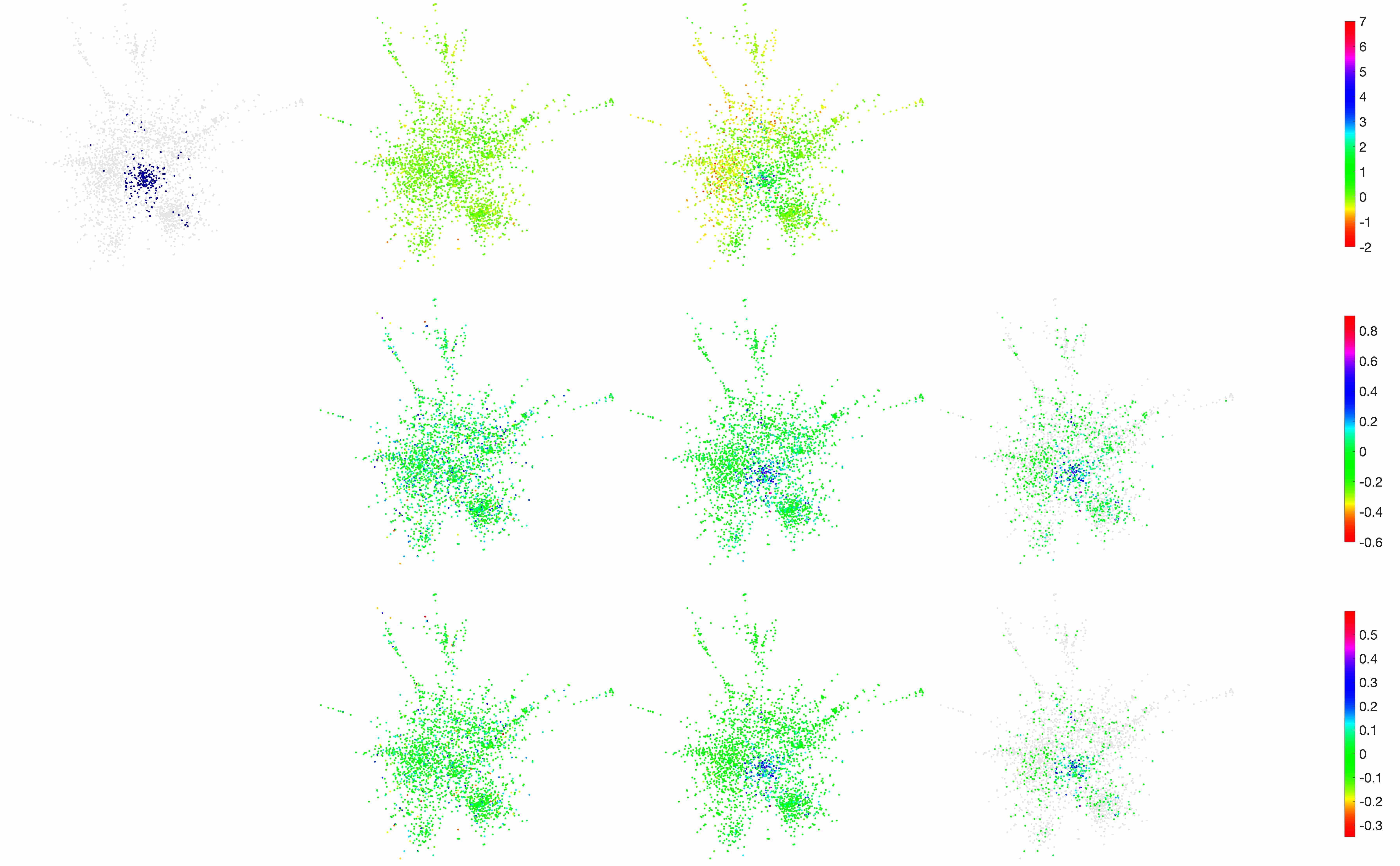}}{1}
    \note{-6.7,1.8}{Ground Truth}
    \note{-2.8,5.5}{Initialization}
    \note{1,5.5}{Before Shrinkage}
    \note{4.8,5.5}{After Shrinkage}
    \draw[ ] (-5,3.6,0) node [rotate=90] {\tiny{Low-Pass}};
    \draw[ ] (-5,0,0) node [rotate=90] {\tiny{High-Pass 1}};
    \draw[ ] (-5,-3.6,0) node [rotate=90] {\tiny{High-Pass 2}};
    \end{annotate}
    \vspace{-7mm}
    \caption{Framelet coefficients on \textbf{Cora}, Class $2$.}
\end{figure*}

\begin{figure*}[t]
    \centering
    \begin{annotate}{\includegraphics[width=\linewidth]{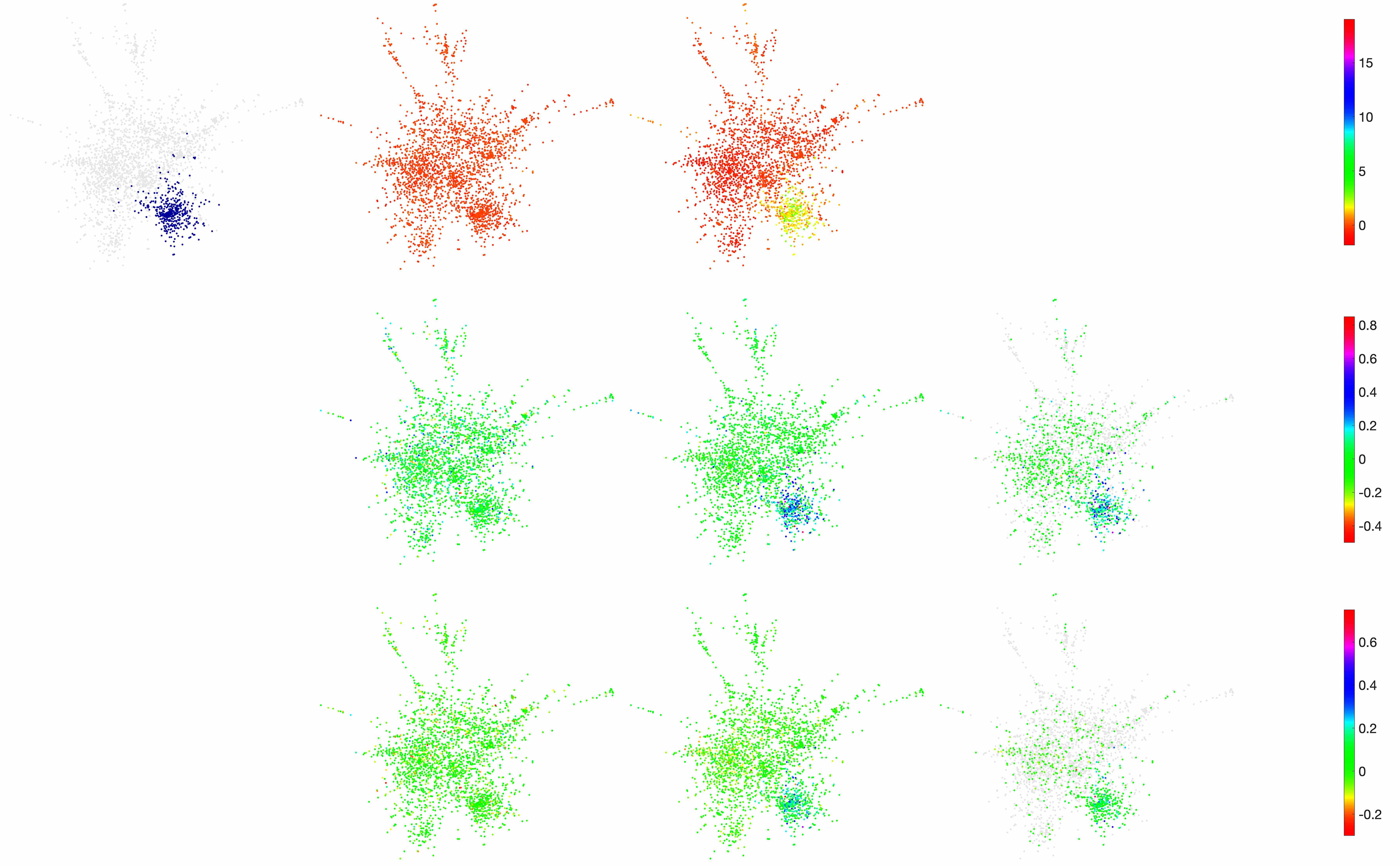}}{1}
    \note{-6.7,1.8}{Ground Truth}
    \note{-2.8,5.5}{Initialization}
    \note{1,5.5}{Before Shrinkage}
    \note{4.8,5.5}{After Shrinkage}
    \draw[ ] (-5,3.6,0) node [rotate=90] {\tiny{Low-Pass}};
    \draw[ ] (-5,0,0) node [rotate=90] {\tiny{High-Pass 1}};
    \draw[ ] (-5,-3.6,0) node [rotate=90] {\tiny{High-Pass 2}};
    \end{annotate}
    \vspace{-7mm}
    \caption{Framelet coefficients on \textbf{Cora}, Class $3$.}
    \vspace{2mm}
    \begin{annotate}{\includegraphics[width=\linewidth]{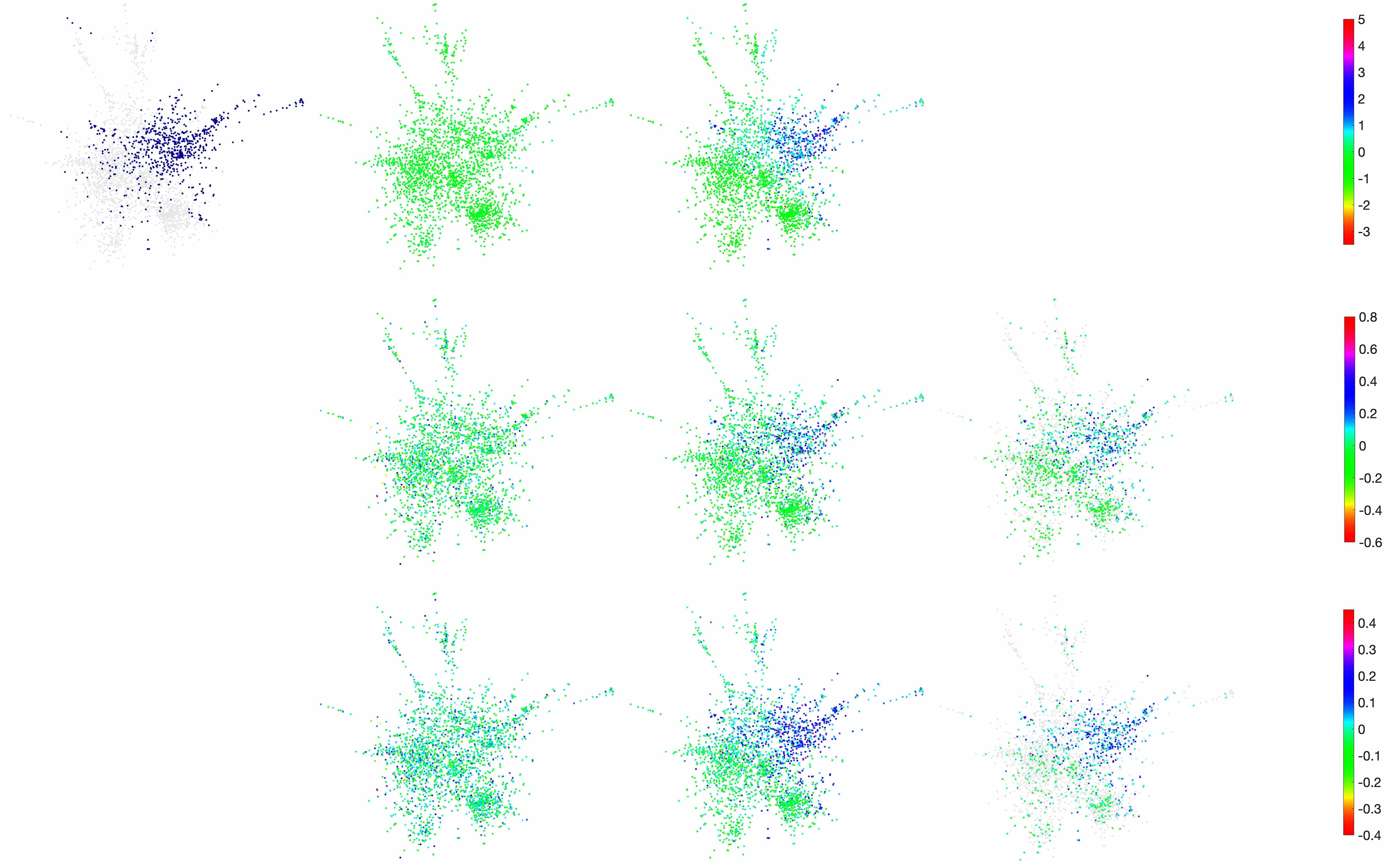}}{1}
    \note{-6.7,1.8}{Ground Truth}
    \note{-2.8,5.5}{Initialization}
    \note{1,5.5}{Before Shrinkage}
    \note{4.8,5.5}{After Shrinkage}
    \draw[ ] (-5,3.6,0) node [rotate=90] {\tiny{Low-Pass}};
    \draw[ ] (-5,0,0) node [rotate=90] {\tiny{High-Pass 1}};
    \draw[ ] (-5,-3.6,0) node [rotate=90] {\tiny{High-Pass 2}};
    \end{annotate}
    \vspace{-7mm}
    \caption{Framelet coefficients on \textbf{Cora}, Class $4$.}
\end{figure*}

\begin{figure*}[t]
    \centering
    \begin{annotate}{\includegraphics[width=\linewidth]{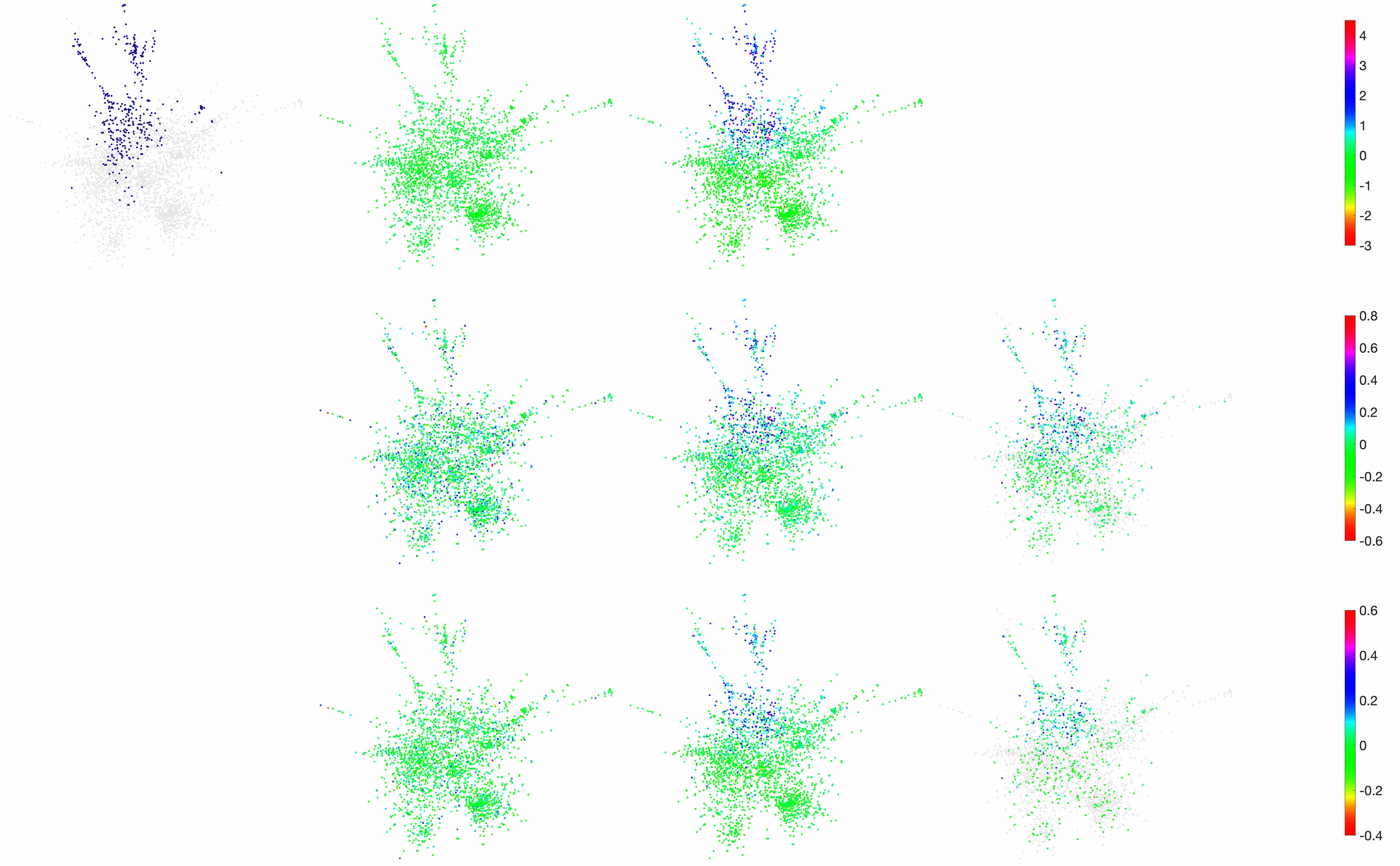}}{1}
    \note{-6.7,1.8}{Ground Truth}
    \note{-2.8,5.5}{Initialization}
    \note{1,5.5}{Before Shrinkage}
    \note{4.8,5.5}{After Shrinkage}
    \draw[ ] (-5,3.6,0) node [rotate=90] {\tiny{Low-Pass}};
    \draw[ ] (-5,0,0) node [rotate=90] {\tiny{High-Pass 1}};
    \draw[ ] (-5,-3.6,0) node [rotate=90] {\tiny{High-Pass 2}};
    \end{annotate}
    \vspace{-7mm}
    \caption{Framelet coefficients on \textbf{Cora}, Class $5$.}
    \vspace{2mm}
    \begin{annotate}{\includegraphics[width=\linewidth]{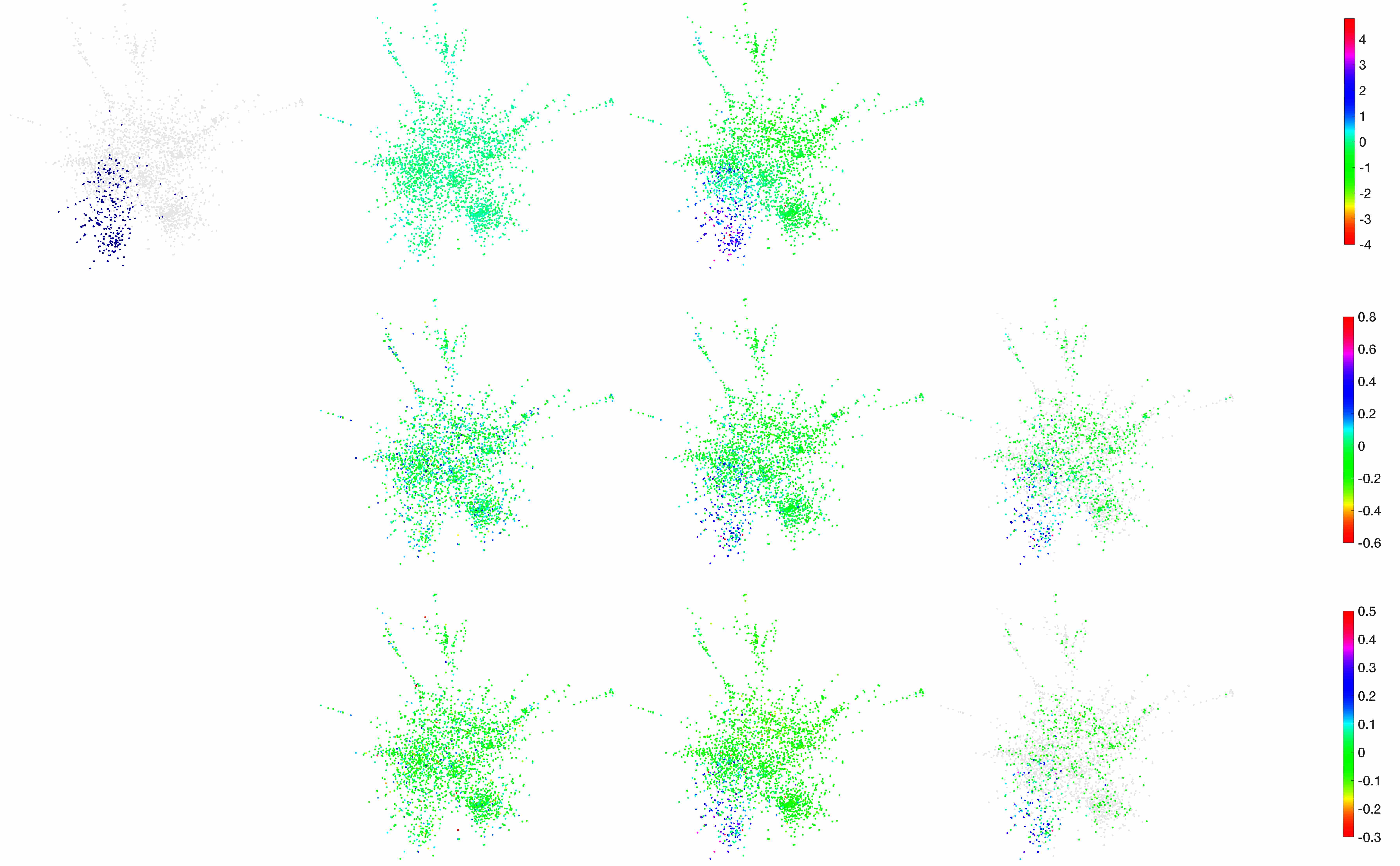}}{1}
    \note{-6.7,1.8}{Ground Truth}
    \note{-2.8,5.5}{Initialization}
    \note{1,5.5}{Before Shrinkage}
    \note{4.8,5.5}{After Shrinkage}
    \draw[ ] (-5,3.6,0) node [rotate=90] {\tiny{Low-Pass}};
    \draw[ ] (-5,0,0) node [rotate=90] {\tiny{High-Pass 1}};
    \draw[ ] (-5,-3.6,0) node [rotate=90] {\tiny{High-Pass 2}};
    \end{annotate}
    \vspace{-7mm}
    \caption{Framelet coefficients on \textbf{Cora}, Class $6$.}
\end{figure*}

\begin{figure*}[t]
    \centering
    \begin{annotate}{\includegraphics[width=\linewidth]{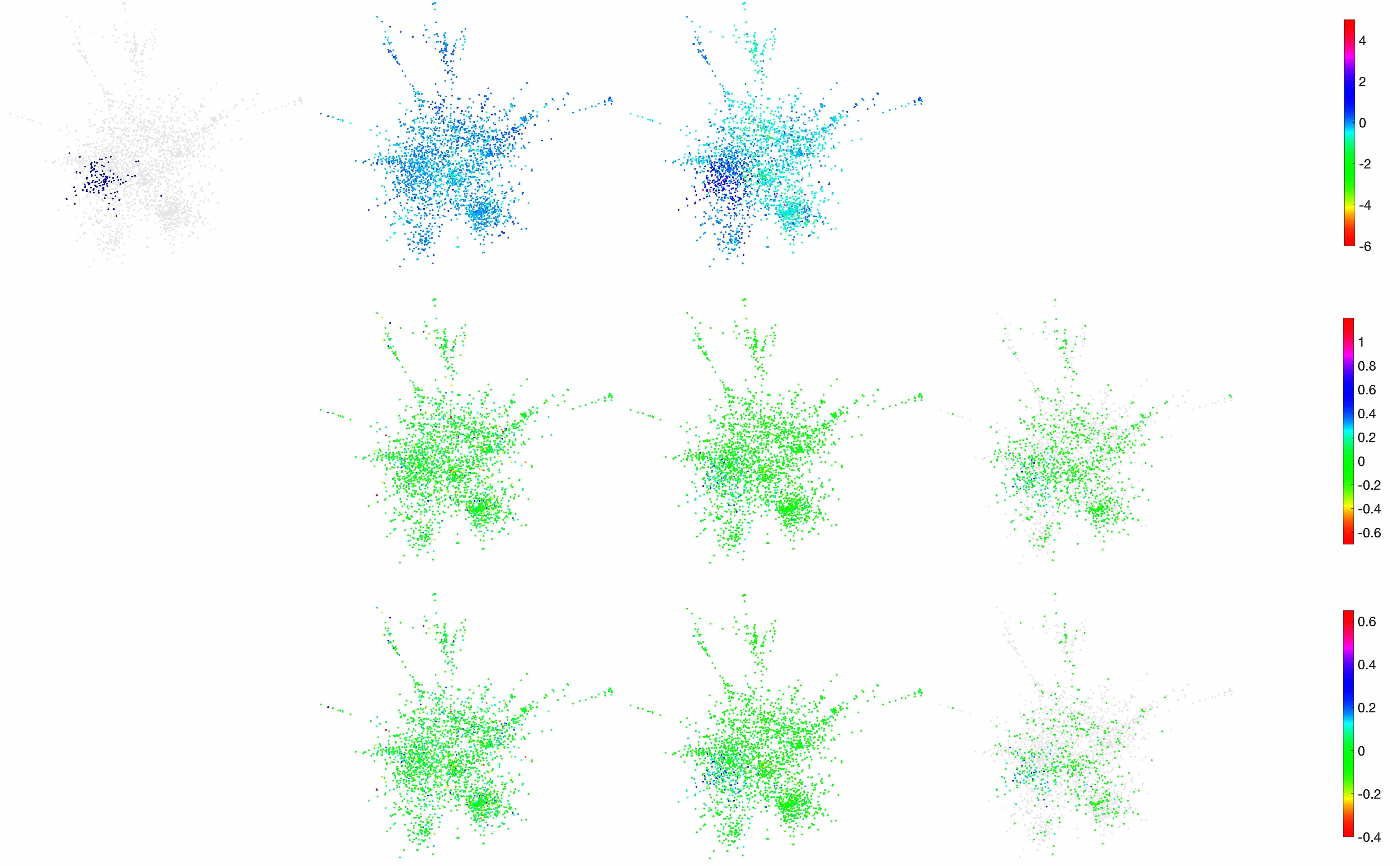}}{1}
    \note{-6.7,1.8}{Ground Truth}
    \note{-2.8,5.5}{Initialization}
    \note{1,5.5}{Before Shrinkage}
    \note{4.8,5.5}{After Shrinkage}
    \draw[ ] (-5,3.6,0) node [rotate=90] {\tiny{Low-Pass}};
    \draw[ ] (-5,0,0) node [rotate=90] {\tiny{High-Pass 1}};
    \draw[ ] (-5,-3.6,0) node [rotate=90] {\tiny{High-Pass 2}};
    \end{annotate}
    \vspace{-7mm}
    \caption{Framelet coefficients on \textbf{Cora}, Class $7$.}
\end{figure*}

\end{document}